\documentclass{article}

\usepackage{ssArxiv_modified}

\usepackage[authoryear]{natbib}

\usepackage[utf8]{inputenc} %
\usepackage[T1]{fontenc}    %
\usepackage{url}            %
\usepackage{enumitem}
\usepackage{booktabs}       %
\usepackage{amsfonts}       %
\usepackage{nicefrac}       %
\usepackage{microtype}      %
\usepackage{xcolor}         %
\usepackage{comment}
\usepackage{titletoc} %

\usepackage{amsmath} 
\usepackage{amssymb}
\usepackage{amsthm}
\usepackage{graphicx}   
\usepackage{adjustbox}
\usepackage{wrapfig}
\usepackage{dsfont}
\usepackage{algorithm}
\usepackage[noend]{algpseudocode}
\usepackage{footnote}

\usepackage{subcaption}

\usepackage{makecell}
\usepackage{multirow}
\usepackage{array}

\usepackage{mathrsfs} 
\usepackage{thm-restate}

\usepackage{xcolor}
\usepackage{tocbibind}
\usepackage{tocloft}
\usepackage{sidecap}
\usepackage{tcolorbox}
\tcbuselibrary{breakable}
\usepackage{dsfont}

\usepackage[hypertexnames=false]{hyperref}       %
\usepackage{cleveref}

\theoremstyle{plain}
\newtheorem{definition}{Definition}[section]
\newtheorem{theorem}{Theorem}[section]  %
\newtheorem{lemma}[theorem]{Lemma}      %
\newtheorem{assumption}[theorem]{Assumption}
\newtheorem{corollary}[theorem]{Corollary}  %
\newtheorem{proposition}[theorem]{Proposition} 

\newtheorem{remark}{Remark}[section]

\DeclareMathOperator*{\argmax}{arg\,max}

\definecolor{darkgreen}{rgb}{0.0, 0.5, 0.0} %

\newcommand{\wcal}{\mathcal{W}}

\newcommand{\zcal}{\mathcal{Z}}
\newcommand{\Rd}{\mathbb{R}^d}

\newcommand{\datadist}{\mu_z^{\otimes n}}

\newcommand{\risk}{\mathcal{R}}
\newcommand{\er}{\widehat{\mathcal{R}}_{S}}

\renewcommand{\paragraph}[1]{{\vspace{0.2mm}\noindent \bf #1}.}

\newcommand{\upperbox}{\overline{\dim}_{\mathrm{B}}}

\makeatletter
\newtheoremstyle{mytheorem}%
  {3pt}%
  {3pt}%
  {}%
  {}%
  {\bfseries}%
  {.}%
  {.5em}%
  {\thmname{#1}\thmnumber{\@ifnotempty{#1}{ }#2}%
   \thmnote{ {\the\thm@notefont(#3)}}}%
\makeatother

\theoremstyle{mytheorem}

\definecolor{extitleclr}{RGB}{206, 106, 53}
\definecolor{exclr}{RGB}{253, 206, 163}

\definecolor{premiseColor}{RGB}{251, 241, 235}
\definecolor{hypColor}{RGB}{244, 251, 253}

\tcolorboxenvironment{problem}{width=\linewidth, sharp corners=all, boxsep=2pt,left=2pt,right=2pt,top=2pt,bottom=2pt, colback=blue!5!white,colframe=black!95!white,breakable}

\tcbuselibrary{theorems}
\newtcbtheorem[number within=section]{example}{Example}{
    theorem style=plain,
    breakable,
    arc=2mm,
    boxrule=.2mm, boxsep=2pt,
    left=2pt,
    right=2pt,
    top=2pt,
    bottom=2pt,
    fonttitle=\normalfont\bfseries,
    colframe=black!15!,
    colback=black!8!,
    coltitle=black,
    before lower={\textcolor{blue!50!black}{\textsf{}}\ },
}{ex}

\newtcbtheorem[use counter=theorem, number within=section]{mytheorem}{Theorem}{
    theorem style=plain,
    breakable,
    arc=2mm,
    boxrule=.2mm,
    boxsep=2pt,
    left=2pt,
    right=2pt,
    top=2pt,
    bottom=2pt,
    fonttitle=\normalfont\bfseries,
    colframe=black!15!,
    colback=black!8!,
    coltitle=black,
    before lower={\textcolor{blue!50!black}{\textsf{}}\ },
}{thm}

\newcommand{\dimS}{\dim_{\mathrm{S}}\mathcal{A}_S}

\newtcbtheorem[use counter=definition, number within=section]{mydef}{Definition}{
    theorem style=plain,
    breakable,
    arc=2mm,
    boxrule=.2mm, boxsep=2pt,
    left=2pt,
    right=2pt,
    top=2pt,
    bottom=2pt,
     colframe=black!15!,
    colback=black!8!,
    coltitle=black,
    fonttitle=\normalfont\bfseries,
    before lower={\textcolor{blue!50!black}{\textsf{}}\ },
}{def}

\title{Generalization at the Edge of Stability}

\author{\name Mario Tuci$^*$ \email{mario.tuci@inria.fr} \\
      \addr{INRIA, CNRS, Département d'Informatique de l'Ecole Normale Supérieure / PSL, France}  \\
      \addr{Department of Computing, Imperial College London, United Kingdom}\\[0.5em]
\name Caner Korkmaz\thanks{$^{,\dagger}$ Authors contributed equally.} \email{c.korkmaz23@imperial.ac.uk} \\
      \addr{Department of Computing, Imperial College London, United Kingdom}\\[0.5em]
\name Umut \c{S}im\c{s}ekli$^\dagger$ \email{umut.simsekli@inria.fr} \\
      \addr{INRIA, CNRS, Département d'Informatique de l'Ecole Normale Supérieure / PSL, France} \\
      [0.5em]
\name Tolga Birdal$^\dagger$ \email{tbirdal@imperial.ac.uk} \\
      \addr{Department of Computing, Imperial College London, United Kingdom} \\
 }

\begin{document}

\maketitle

\begin{abstract}
Training modern neural networks often relies on large learning rates, operating at the \emph{edge of stability}, where the optimization dynamics exhibit oscillatory and \emph{chaotic behavior}. Empirically, this regime often yields improved generalization performance, yet the underlying mechanism remains poorly understood. In this work, we represent stochastic optimizers as \emph{random dynamical systems}, which often converge to a \emph{fractal attractor set} (rather than a point) with a smaller intrinsic dimension. Building on this connection and inspired by Lyapunov dimension theory, we introduce a novel notion of dimension, coined the `sharpness dimension', and prove a generalization bound based on this dimension. Our results show that generalization in the chaotic regime depends on the \emph{complete} Hessian spectrum and the structure of its \emph{partial determinants}, highlighting a complexity that cannot be captured by the trace or spectral norm considered in prior work. %
Experiments across various MLPs and transformers validate our theory while also providing new insights into the recently observed phenomenon of grokking.
\vspace{-4mm}

\end{abstract}

\vspace{-1.5mm}\section{Introduction}\vspace{-1.5mm}
Understanding why large, overparameterized neural networks trained by gradient-based methods generalize remains one of the central open problems in modern machine learning. 
\begin{figure}[t]
\centering\includegraphics[width=\columnwidth]{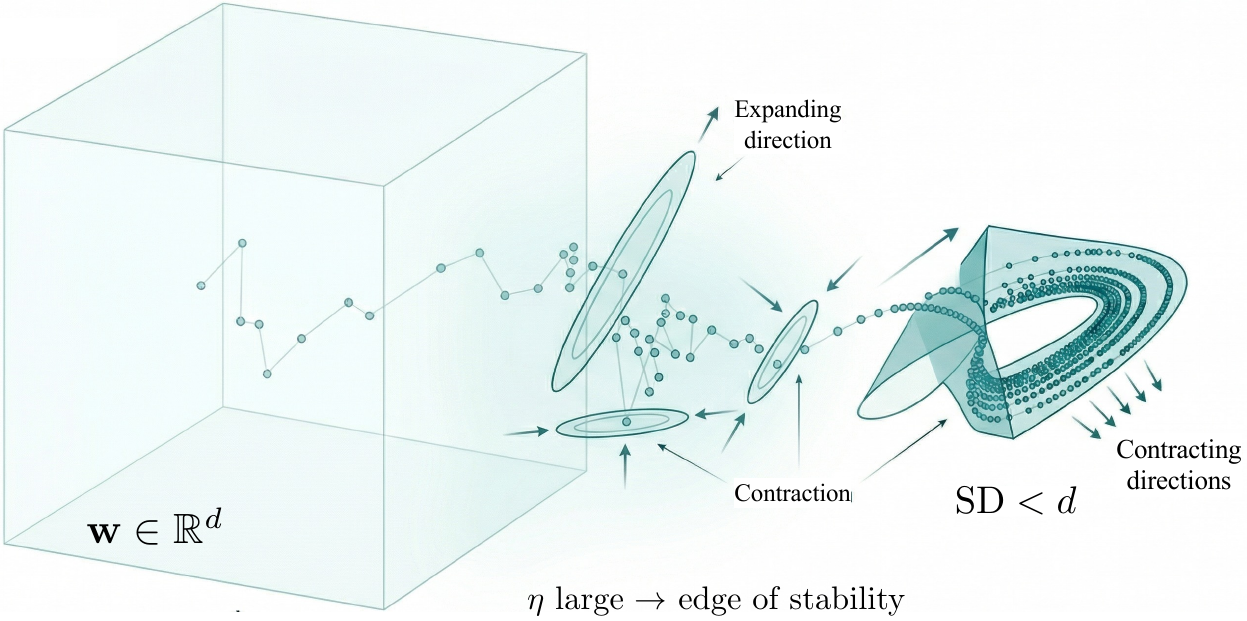}
\caption{\textbf{Generalization at the Edge of Stability (EoS).} Modeling stochastic optimization as a random dynamical system (RDS), we show that at EoS the leading sharpness satisfies $\lambda_1>0$, implying expansion along at least one direction.  The fundamental balance between expansion and contraction implies that the effective dimensionality of the dynamics, measured by our \textbf{Sharpness Dimension} (SD), is strictly smaller than the ambient parameter space: $\mathrm{SD}<d$. We prove that the \textbf{worst-case generalization error is governed by $\mathrm{SD}$ rather than the parameter count}. Our results identify EoS as precisely the regime where generalization is controlled by a provably lower-dimensional attractor, providing a principled explanation for why overparameterized models can generalize beyond classical complexity measures.} 
\label{fig:GATES}
\vspace{-13mm}
\end{figure}
Recent empirical evidence has revealed a phenomenon that challenges classical convex optimization theory. \cite{cohen2021gradient} observed that when training neural networks with gradient descent (GD) with a fixed learning rate~$\eta$, the largest eigenvalue of the loss Hessian often oscillates around, and frequently exceeds $2/\eta$, even as the training loss continues to decrease. This behavior, termed the \emph{edge of stability} (EoS), has generated considerable interest (see Sec.~\ref{sec:related_work}), since the threshold $2/\eta$ implies instability and divergence for quadratic objectives \cite{ghosh2025variational}.
\citet{ly2025optimization} demonstrated that exceeding the threshold of $\frac{2}{\eta}$ is sufficient to induce chaotic training dynamics. Additionally, in the chaotic regime the optimizer will not settle at a single point, rather it explores a bounded, typically fractal like set \cite{singh2023universal}.
This raises a fundamental question:

\emph{How can generalization be explained in the regime that is locally unstable and potentially chaotic?}

A natural response to this puzzle is to examine the local geometry of the loss landscape. In particular, the Hessian, which encodes local curvature, has long been viewed as a key lens for understanding generalization \citep{keskar2016large, jiang2019fantastic},  motivating a large body of work based on \emph{pointwise} notions of \emph{sharpness} and \emph{flatness} (see Sec.~\ref{sec:related_work_hessian}). Despite its appeal and empirical successes, this viewpoint has been challenged. It is now well understood that sharpness-based criteria (e.g., Hessian trace) are neither necessary nor sufficient for good generalization: there exist flat minima that generalize poorly and sharp that do well \citep{dinh2017sharp, kaur2023maximum, wen2023sharpness}.

Consequently, characterizing generalization at the edge of stability through pointwise analysis of individual solutions may be fundamentally inadequate. For practical learning rates, training dynamics are expected to exhibit chaotic behavior \citep{singh2023universal}, wherein training trajectories display sensitive dependence on initialization. In this regime, generalization performance should be attributed not to the properties of any single solution, but rather to the geometric and characteristics of the entire solution set explored by the optimizer in the long term. %

\paragraph{Contributions} We introduce a framework for studying generalization at the edge of chaos where modern deep networks operate~\citep{cai2026does}. In particular, we contribute the following:
\begin{itemize}[noitemsep,topsep=0pt,leftmargin=*]
    \item \textbf{Attractor-Centric Framework}: Modeling stochastic optimization as a random dynamical system (RDS), we shift the study of generalization at EoS from isolated parameter vectors to the \textbf{geometric properties of the (random) attractor}.
    \item \textbf{RDS Sharpness \& Sharpness Dimension ($\mathrm{SD}$)}: We propose \textbf{two new complexity measures}, RDS Sharpness \& $\mathrm{SD}$, derived not from the trajectory but from expansion and contraction rates that characterizes the attractor's geometry.
    \item \textbf{Generalization Bound}: We provide \textbf{a new bound on the worst case generalization error}, rigorously linking it to the fractal dimension of the Random Attractor measured as $\mathrm{SD}$.
    \item \textbf{Empirical Validation}: We explain how to \textbf{compute SD efficiently} and deploy our findings in quantifying generalization across multilayer perceptrons and \textbf{recent transformers} (GPT-2~\citep{radford2019language}) as well as to study the recently introduced paradigm of \emph{grokking}, delayed and sudden generalization~\citep{power2022grokking,prieto2025grokking}. %
\end{itemize}

As illustrated in Fig.~\ref{fig:GATES}, our results identify EoS as precisely the regime where generalization is controlled by a provably lower-dimensional attractor, whose effective dimensionality is controlled by SD and is strictly smaller than the ambient parameter space. We prove that in the EoS, the {worst-case generalization error} is governed exactly by $\mathrm{SD}$. Our findings present a principled explanation for why overparameterized models can generalize beyond classical complexity measures. Our code will be publicly available under: \href{https://circle-group.github.io/research/GATES}{https://circle-group.github.io/research/GATES}. 

\vspace{-2mm}
\section{Related Work}
\label{sec:related_work}
\vspace{-1mm}
\paragraph{Generalization bounds}
Recent work has established strong empirical and theoretical connections between generalization and geometric and topological complexity measures derived from the \emph{optimization trajectory} \citep{simsekli2020hausdorff, birdal2021intrinsic, dupuis2023generalization, andreeva2024topological, tuci2025mutual}. In particular, it has been shown that  weight iterates sampled \emph{after} initial convergence encode critical information about generalization. Topological summaries such as the `$\alpha$-weighted lifetime sum' exhibit consistent correlations with the generalization gap across training runs \citep{andreeva2024topological, tuci2025mutual}. These findings suggest that the stochastic fluctuations observed during late-stage training are not merely noise, but reflect the optimizer exploring a structured geometric object. Differently, \citet{camuto2021fractal} showed that under contraction, stochastic GD (SGD) admits an invariant measure supported on a fractal set and linked its geometric complexity to generalization. However, this theory relies on contraction and therefore does not apply at the EoS.

\paragraph{Hessian and generalization} 
\label{sec:related_work_hessian}
Since \citet{hochreiter1994simplifying}, it has been conjectured that flatter minima generalize better, motivating extensive work on flatness and sharpness \citep{keskar2016large,dinh2017sharp,neyshabur2017exploring,sagun2017empirical,yao2018hessian,chaudhari2019entropy,simsekli2019tail,nguyen2019first,mulayoff2020unique,tsuzuku2020normalized,ahn2023escape}. Despite this effort, there is no universally accepted definition of flatness, and most practical surrogates rely on second-order cues such as the trace of the Hessian \citep[e.g.,][]{jastrzebski2020break,wen2023sharpness}.

In parallel, optimization methods explicitly designed to favor flat minima have shown strong empirical gains in generalization \citep{izmailov2018averaging,wu2020adversarial,zheng2021regularizing,kaddour2022flat,foret2020sharpness}. Theoretically, minimizing the Hessian trace was shown to select the true solution in low-rank matrix recovery \citep{ding2024flat}, extended to deep networks \citep{gatmiry2023inductive}, and supported empirically for large language models \citep{liu2023same}, and linked to output stability \citep{ma2021linear}. More recently, generalization bounds have been obtained in terms of sums of gradient norms \citep{haddouche2024pac,clerico2022generalisation}, which reduce to the Hessian trace under suitable conditions. Despite these advances, recent results show that trace-based flatness alone does not guarantee good generalization \citep{wen2023sharpness}.

\paragraph{Edge of Stability (EoS) and chaos}
The dynamics of deep network training have been widely studied, with early work documenting rapid changes in the local loss landscape during the initial phase of optimization \citep{keskar2016large,jastrzebski2020break,xing2018walk}, and later characterizing it for gradient descent by \citet{cohen2021gradient}. A growing literature has since been interested in EoS phenomena \cite{arora2022understanding,ahn2023learning,ahn2022understanding,wang2022analyzing,chen2023beyond,damian2022self,zhu2022understanding}. In particular, \citet{ahn2022understanding,ma2021linear} showed that the existence of a forward-invariant set prevents divergence, and that for $\tanh$ networks such a set exists, explaining stability even at large learning rates in the EoS regime. Related phenomena have also been observed for SGD via the notion of mini-batch sharpness \citep{andreyev2024edge}. From a complementary perspective, stability analyses reveal connections to chaos \citep{sasdelli2021chaos,ly2025optimization}: in particular, \citet{ly2025optimization} show that sustained criticality of the top Hessian eigenvalue (EoS) is sufficient to induce chaos, and \citet{chemnitz2025characterizing} study a cubic model to characterize the boundary between EoS and divergence.

Our work connects Edge of Stability, chaotic dynamics, Hessian-based generalization bounds via rigorous generalization guarantees through the lens of random dynamical systems. 

\vspace{-2mm}
\section{Preliminaries}
\vspace{-2mm}
\paragraph{Learning setup}
We consider supervised learning with parameter vector $w \in \mathbb{R}^d$ and population risk $\mathcal{R}(w) := \mathbb{E}_{Z \sim \mu}[\ell(w,Z)]$,
where $\mu$ is an unknown data distribution and $\ell: \mathbb{R}^d \times\mathcal{Z}$ is the composed loss function. In practice, training proceeds by minimizing the empirical risk $\er(w) := \frac{1}{n} \sum_{i=1}^n \ell(w,Z_i)$ over a dataset $S=\{Z_i\}_{i=1}^n \sim \datadist$ using stochastic gradient methods. Our analysis focuses on the asymptotic behavior of the optimization dynamics rather than on a single training iterate. Therefore we are interested in the worst-case generalization gap 
\begin{align}
    \label{eq:worst_case_gap}
    \mathcal{G}(\mathcal{A}_S(\omega)) :=\sup_{w \in \mathcal{A}_{S}(\omega)} \risk(w) - \er(w).
\end{align}
Here, $\mathcal{A}_S(\omega)$ denotes a dataset-dependent random set, where $S \in \mathcal{Z}^n$ is the dataset and $\omega$ is an independent random variable capturing algorithmic noise (e.g., minibatch sampling). In particular, $\mathcal{A}(\omega)$ denotes the \emph{random attractor} associated with the optimizer, to be precised in Dfn~\ref{def:random_attractor}.

\paragraph{Random dynamical systems}
To have a rigorous consideration of random attractors as sets, we represent  stochastic optimization algorithms as discrete-time \emph{random dynamical systems} (RDS) according to \citet{arnold2006random}. While the definition is abstract, it will follow with a concrete example for SGD, represented in the RDS form. Note that a similar formalism was introduced to study the stability of SGD in \citet{chemnitz2025characterizing}.
\begin{definition}[Random Dynamical System]
\label{def:rds}
A \emph{discrete-time random dynamical system (RDS)} on $\mathbb{R}^d$ is a tuple $(\Omega, \mathcal{F}, \mathbb{P}, \theta, \phi)$ consisting of the following components:
\begin{enumerate}[leftmargin=*,topsep=0.01em,itemsep=0.01em]
    \item A \textbf{metric dynamical system} $(\Omega, \mathcal{F}, \mathbb{P}, \theta)$, where $(\Omega, \mathcal{F}, \mathbb{P})$ is a probability space and $\theta: \Omega \to \Omega$ is an invertible, measure-preserving, and ergodic transformation\footnote{A transformation $\theta$ is \textit{measure-preserving} if $\mathbb{P}(\theta^{-1}A) = \mathbb{P}(A)$ for all $A \in \mathcal{F}$, and \textit{ergodic} if any invariant set $A = \theta^{-1}A$ satisfies $\mathbb{P}(A) \in \{0, 1\}$.} such that the map $(t, \omega) \mapsto \theta^t (\omega)$ defines a $\mathbb{Z}$-action on $\Omega$. That is, the family of maps $\{\theta^t\}_{t \in \mathbb{Z}}$ satisfies
    \[
    \theta^0 = \mathrm{id}_\Omega\footnote{For an arbitrary set $A$, $\mathrm{id}_{A} \colon A \to A$ denotes the identity map defined by $\mathrm{id}_{A}(x) = x$ for all $x \in A$.}, \qquad \theta^{t+s} = \theta^t \circ \theta^s \quad \text{for all } t,s \in \mathbb{Z},
    \]
where we use the notation $\theta\omega := \theta(\omega)$. %
    \item A \textbf{measurable cocycle} $\phi: \mathbb{N}_0 \times \Omega \times \mathbb{R}^d \to \mathbb{R}^d$ over $\theta$, which is $(\mathcal{B}(\mathbb{N}_0) \otimes \mathcal{F} \otimes \mathcal{B}(\mathbb{R}^d), \mathcal{B}(\mathbb{R}^d))$\footnote{Let $\mathcal{B}(\mathbb{R}^d)$ or $\mathcal{B}(\mathbb{N}_0)$ denote the corresponding Borel $\sigma$-algebras.}
-measurable and satisfies the \emph{cocycle property}:
    \begin{subequations}
    \begin{align}
        &\phi(0, \omega, \cdot) = \text{Id}_{\mathbb{R}^d}, \\
        &\phi(t+s, \omega, w) = \phi(t, \theta^s \omega, \phi(s, \omega, w)),
        \label{eqn:mds_b}
    \end{align}
    \end{subequations}
    for all $t, s \in \mathbb{N}_0$, $\omega \in \Omega$, and $w \in \mathbb{R}^d$.
\end{enumerate}
An RDS is said to be $k$ times \emph{continuously differentiable} (or $C^k$) if the mapping $w \mapsto \phi(t, \omega, w)$ is $C^k$ for all $t \in \mathbb{N}_0$ and $\omega \in \Omega$.
\end{definition}
Let us provide a more intuitive explanation for the above definition by considering SGD as a special case. 
Given a dataset $S = \{Z_1, \dots, Z_n\} \in \mathcal{Z}^n$, SGD is based on the following recursion:
\begin{align}
\label{eqn:sgd}
    w_{t+1} = w_t - \eta \left( \frac{1}{b} \sum_{i\in \Omega_{t}} \nabla \ell(w_t,Z_i) \right),
\end{align}
where $\Omega_t \subset \{1,\dots,n\}$ is the minibatch with $|\Omega_t| = b $ being the batch-size. Since there are only finitely many possible choices of minibatches, i.e., $\binom{n}{b}$ many, we can enumerate all the minibatches such as $\{\Omega^{(1)}, \dots, \Omega^{(\binom{n}{b})}\}$ s.t. $\Omega_t = \Omega^{(j)}$ for some $j$. Hence \eqref{eqn:sgd} can be alternatively written as:
\begin{align}
\label{eqn:sgd_alt}
    w_{t+1} = w_t - \eta \left( \frac{1}{b} \sum_{i\in \Omega^{(j_{t})}} \nabla \ell(w_t,Z_i) \right),
\end{align}
where $j_t$ is randomly drawn from the set $\{1,\dots, \binom{n}{b}\}$. 
\vspace{3pt}

Going back to Definition~\ref{def:rds}, the random event $\omega$ corresponds to the \emph{algorithmic randomness}. In the case of SGD, the only source of algorithmic randomness is the choice of minibatches at every iteration. Hence, for SGD, $\omega = \{\dots, j_{-t},\dots j_{-1}, j_0, j_1, \dots, j_t, \dots \}$ will encapsulate the infinite sequence of minibatch indices that are drawn through optimization\footnote{For technical reasons, we need to consider a doubly infinite sequence that also takes into account for negative times $-t$.}. In other words, a single event $\omega$ will contain all the information about the randomness coming from the algorithm.

Given a sequence of minibatch indices $\omega$, we can define the \emph{RDS map} $\phi$, which essentially corresponds to the algorithm update rule. For $t=1$, we have the following form:
\begin{align}
\label{eqn:phi1}
    \phi(1, \omega, w) :=  w - \eta \left( \frac{1}{b} \sum_{i\in \Omega^{(j_{0})}} \nabla \ell(w,Z_i) \right),
\end{align}
which coincides with \eqref{eqn:sgd_alt}, but only for $t=1$. 
\vspace{3pt}

At iteration $t$, the stochastic gradient is computed over the minibatch with index $j_{t-1}$. Hence, to define the RDS map for a general $t$, i.e. $\phi(t,\omega, w)$, we additionally need to extract the minibatch index $j_{t-1}$ from the infinite sequence $\omega$. This operation will be done by the \emph{metric dynamical system} $\theta$. For this example\footnote{For most of the stochastic optimizers one can use the same metric dynamical system as long as the algorithmic randomness is only coming from random minibatches.}, we set $\theta: \mathbb{Z} \times \Omega \to \Omega$ to the so-called the `left-shift operator', $(\theta \omega)_k := j_{k+1}$, which takes an infinite sequence $\omega$ and shifts its elements by one coordinate, and returns the resulting infinite sequence. Iterating this operator $t$ times would shift the coordinates $t$ times: i.e., $(\theta^t \omega)_k := j_{k+t}$. 
\vspace{3pt}

Given all the ingredients, by \eqref{eqn:mds_b} we can define our RDS for a general $t$:
\begin{align}
\label{eqn:phit}
    \phi(t, \omega, w) = \phi(1, \theta^{t-1} \omega, \phi(t-1, \omega, w)),
\end{align}
where $\phi(1,\cdot,\cdot)$ is defined in \eqref{eqn:phi1}. It is now easy to verify that \eqref{eqn:phit} exactly recovers the recursion given in \eqref{eqn:sgd_alt}.
Finally, we observe that the system satisfies the cocycle property:
\begin{align*}
    \phi(t, \omega, w) = \phi(1, \theta^{t-1}\omega, \cdot) \circ \dots \circ \phi(1, \omega, w), \forall t \in \mathbb{N}_0.
\end{align*}
This property serves as a fundamental consistency requirement. Intuitively, it ensures that evolving a state for $t+s$ steps is equivalent to evolving it for $s$ steps and then resuming for $t$ more steps using the remaining noise history $\theta^s \omega$: system's evolution is chronologically coherent.

While denoting an optimization algorithm with such abstract notions might seem rather unorthodox, thanks to this formal connection, we will be able to access the rich toolbox of random dynamical systems theory.

\paragraph{Random Attractors}
Dynamics driven by stochastic optimization with persistent noise (e.g. constant learning rate SGD) generally do not converge to a single location. We are interested in the set in which an RDS settles. 
\begin{definition}[Pullback Random Attractor]
\label{def:random_attractor}
Let $(\Omega, \mathcal{F}, \mathbb{P}, \theta, \phi)$ be a discrete-time RDS according to Dfn.~\ref{def:rds}. A mapping $\omega \mapsto \mathcal{A}(\omega)$ from $\Omega$ into the space of non-empty compact subsets of $\mathbb{R}^d$ (denoted $\operatorname{comp}(\mathbb{R}^d)$ ) is called a \emph{pullback random attractor} if it is $(\mathcal{F}, \mathcal{B}(\operatorname{comp}(\mathbb{R}^d)))$-measurable and satisfies the following two properties for $\mathbb{P}$-almost all $\omega \in \Omega$: 
\begin{enumerate}[leftmargin=*,topsep=0em,itemsep=0.01em]
    \item \textbf{Invariance:} $\phi(t, \omega, \mathcal{A}(\omega)) = \mathcal{A}(\theta^t \omega)$ for all $t \in \mathbb{N}_0$.
    \item \textbf{Pullback Attraction:} For every deterministic bounded set $B \subset \mathbb{R}^d$:
    \begin{equation}
        \lim_{t \to \infty} \operatorname{dist}\bigl( \phi(t, \theta^{-t}\omega, B), \mathcal{A}(\omega) \bigr) = 0,
    \end{equation}
    where $\operatorname{dist}(X, Y) := \sup_{x \in X} \inf_{y \in Y} \|x - y\|$ is the Hausdorff semi-distance. Further the image of a function is defined as follows. If $f : X \to Y$ is a map and $A \subset X$, then the image of $A$ under $f$ is defined as $f(A) := \{\, f(x) \in Y \;:\; x \in A \,\}$.
\end{enumerate}
\end{definition}
In stochastic systems, attraction must be understood in a \emph{pullback} rather than forward sense. Because fresh noise is continually injected, trajectories that come close can later separate, so pathwise forward convergence generally fails. The pullback viewpoint instead fixes a noise realization and examines 
the state at time $t=0$ obtained by initializing the system in the remote past. The resulting pullback 
attractor is a noise-conditioned ``snapshot'' of the asymptotic state, representing the set to which all past histories converge at the present time. This is the central object of our analysis.

\vspace{-2mm}
\section{Theoretical Results}
\vspace{-1mm}
\paragraph{Sharpness}
Recent work introduced the notion of sharpness in terms of the Hessian of the empirical risk \cite{cohen2021gradient}. We introduce an alternative notion of sharpness that extends to general random dynamical systems and provides an intuitive understanding of how the two notions are related. 
\begin{mydef}{RDS Sharpness}{rds:sharpness}~
Let $(\Omega, \mathcal{F}, \mathbb{P}, \theta, \phi)$ be a $C^1$ random dynamical system on $\mathbb{R}^d$. For a fixed state $w \in \mathbb{R}^d$, let $D\phi(1,\omega,w) \in \mathbb{R}^{d \times d}$ denote the Jacobian (Fr\'echet derivative) 
of the map $x \mapsto \phi(1,\omega,x)$ evaluated at $x = w$. Let $\mathcal{A}(\omega)$ be random compact set. We define the \textbf{RDS Sharpness of Order $k$} as the expected log-variation of the $k$-th singular value:
\begin{equation}
\label{eqn:rds_sharpness}
    \lambda_k := \mathbb{E} \left[ \sup_{w \in \mathcal{A}(\omega)}\ln \sigma_k(\omega, w) \right] \quad \forall k \in \{1, \dots, d\},
\end{equation}
where $\sigma_1(\omega, w) \geq \dots \geq \sigma_d(\omega, w)$ are the singular values of $D\phi(1, \omega, w)$, assuming integrability holds.
\end{mydef}

The next example, shows how Dfn.~\ref{def:rds:sharpness} relates to the classical notion of sharpness introduced in \citet{cohen2021gradient}, which is defined in terms of the largest eigenvalue 
of the Hessian.

\begin{example}{GD Sharpness}{gd_sharpness}~
\\
We can interpret GD as a deterministic instance of a random dynamical system where the algorithmic randomness $\omega$ is constant. Let $\er(w)$ denote the empirical risk for a fixed dataset $S \in \zcal^n$ and $w \in \mathbb{R}^d$; the discrete-time update map $\phi$ is defined as:
\begin{equation}
\label{eq:gd_cocycle}
    \phi(1, w) = w - \eta \nabla \er(w).
\end{equation}
The local stability of this system is governed by the Jacobian $D\phi(1, w) = I - \eta \nabla^2 \er(w)$. In the optimization literature \citep{cohen2021gradient}, sharpness is traditionally defined as the largest eigenvalue  the Hessian $\nabla^2 \er(w)$. \\
\\
If we interpret our sharpness definition (see Dfn.~\ref{def:rds:sharpness}) locally over point (e.g. $\mathcal{A}(\omega) = \{w\}$), we obtain, that the sharpness of order 1 corresponds to $\lambda_1 = \ln (\| I - \eta \nabla^2 \er(w)\|) $, which exhibits related information to the largest eigenvalue of the hessian. Indeed in the EoS regime, we have $\lambda_1 \geq 0$.
\end{example} 
\paragraph{Edge of Stability and Chaos}
To build intuition for the relationship between our notion of RDS-sharpness of order $1$ (Dfn.~\ref{def:rds:sharpness}), the classical sharpness of \citet{cohen2021gradient}, and the onset of chaos, we consider the linearized framework in \citet{ly2025optimization}, and examine GD in the EoS regime through its \emph{sensitivity to initial conditions}. Let $\phi(K, w)$ denote the $K$-step GD cocycle (cf. \eqref{eq:gd_cocycle}) and consider a reference trajectory $x_t = \phi(t, w_0) = \phi(1, w_{t-1})$. To track the propagation of an infinitesimal perturbation $\delta w_t$, we linearize the dynamics:
\begin{equation}
    \delta w_{t+1} := \phi(1,w_t+ \delta w_t) - \phi(1,w_t)   \approx D\phi(1, w_t) \delta w_t,\nonumber
\end{equation}
where $D\phi(1, w_t) = I - \eta \nabla^2 \er(w_t)$ is the \emph{one-step Jacobian} (cf. Ex.~\ref{ex:gd_sharpness}). The local exponential growth rate in the direction $v_t = \delta w_t / \|\delta w_t\|$ is given by:
\begin{equation}
    \log \frac{\| \delta w_{t+1} \|}{\| \delta w_t \|} = \log \| (I - \eta \nabla^2 \er(w_t)) v_t \|.
\end{equation}

Let $\{(\alpha_i(t), u_i(t))\}_{i=1}^d$ denote the eigenpairs of the Hessian $\nabla^2 \er(w_t)$. If the perturbation direction $v_t$ is aligned with an eigenvector $u_i(t)$, the one-step growth factor is $|1 - \eta \alpha_i(t)|$. In the EoS regime, the maximum eigenvalue $\alpha_{\max}(t)$ typically exceeds the stability threshold $2/\eta$. When $\alpha_{\max}(t) > 2/\eta$, the growth factor satisfies $|1 - \eta \alpha_{\max}(t)| > 1$, or equivalently, $\ln |1 - \eta \alpha_{\max}(t)| > 0$. 
In the context of Ex.~\ref{ex:gd_sharpness} this would make the singular value $\sigma_1$ in \eqref{eqn:rds_sharpness} larger than $0$, hence we obtain $\lambda_1 > 0$ for this system.

If this condition holds on average along the trajectory, the top Lyapunov exponent $ \Lambda_1 := \lim_{T \to \infty} \frac{1}{T} \sum_{t=0}^{T-1} \ln \| D\phi(1, x_t) v_t \|$  see \citet[Lemma~3.2.2, p.~113]{arnold2006random} becomes positive . In the language of dynamical systems, $\Lambda_1 > 0$ implies that trajectories diverge from each other exponentially (even though the system might not be divergent), providing a signature of \emph{deterministic chaos} within the EoS oscillations. Hence, we argue that EoS emerges when the system is chaotic with $\lambda_1 > 0$ and our goal is to develop theoretical tools specifically designed for this challenging setup. 
\begin{remark}[Sharpness for an RDS]
\citet{andreyev2024edge} observed that the \emph{expected} mini-batch Hessian typically concentrates near the stability threshold $2/\eta$. 
\end{remark}

\paragraph{Existence of Random Attractor}
Since our work focuses on the random pullback attractor to which the RDS settles, we briefly discuss the existence of the random attractor.
\begin{proposition}[Existence of the Random Pullback Attractor \citep{crauel1997random}]
\label{prop:existence_attractor}
Let $(\Omega, \mathcal{F}, \mathbb{P}, \theta, \phi)$ be a $C^0$ random dynamical system on $\mathbb{R}^d$. Suppose there exists a bounded random set $K(\omega)$ that is \textbf{pullback absorbing} for $\mathbb{P}$-almost every $\omega \in \Omega$; that is, for every deterministic bounded set $D \subset \mathbb{R}^d$, there exists a time $T(D, \omega) \geq 0$ such that
\begin{equation}
    \phi(t, \theta^{-t}\omega, D) \subset K(\omega) \quad \forall t \geq T(D, \omega).
\end{equation}
Then, there exists a unique, compact and measurable \footnote{Measurable with respect to the \textbf{past $\sigma$-algebra} $\mathcal{F}_{-\infty}^0 := \sigma \{ \theta^t \omega : t \leq 0 \}$.} random pullback attractor $\mathcal{A}(\omega)$ satisfying Dfn.~\ref{def:random_attractor}. 
\end{proposition}
Prop.~\ref{prop:existence_attractor} has a clear dynamical meaning.
The existence of a pullback absorbing set \(K(\omega)\) means that, for a fixed noise realization \(\omega\), all trajectories eventually enter a bounded region of the state space, provided the system is evolved from the distant past to the present.
The random pullback attractor \(\mathcal{A}(\omega)\) is the smallest invariant set inside this region.
It contains exactly the states that can be reached asymptotically under the fixed noise realization \(\omega\).

Indeed, this demonstrates that local instability, when coupled with global dissipativity, leads to the emergence of a compact pullback attractor. An intuitive explanation of this mechanism, for example in the case of neural networks using $\tanh$ as activation function, is given by \citet{ahn2022understanding}.

\paragraph{Complexity meassure of Random Attractors}
Given the existence of the random pullback attractor, a natural question arises: Can we quantify its complexity, particularly in the EoS regime, where we expect chaos.  %

\begin{mydef}{Sharpness Dimension}{local_sharpness_unified}~
Let $(\Omega, \mathcal{F}, \mathbb{P}, \theta, \phi)$ be a $C^1$ discrete random dynamical system on $\mathbb{R}^d$. Further let $\mathcal{A}(\omega) \subset \mathbb{R}^d$ be an almost surely compact random set. As in Dfn.~\ref{def:rds:sharpness}, let $\lambda_k$ be the RDS sharpness of order $k$, for $k=1, \dots, d$.

Set
\[
j^* := \max \left\{ i \in \{1, \dots, d\} : \sum_{k=1}^i \lambda_k \ge 0 \right\},
\]
with the convention $j^*=0$ if $\lambda_1 < 0$.

The \textbf{Sharpness Dimension of } $\mathcal{A}(\omega)$ is defined as
\[
\dim_{\mathrm{S}}\mathcal{A} :=
\begin{cases}
j^* + \frac{\sum_{i=1}^{j^*} \lambda_i}{|\lambda_{j^*+1}|}, & \text{if } 1 \le j^* < d, \\
d, & \text{if } j^* = d, \\
0, & \text{if } \lambda_1 < 0.
\end{cases}
\]
\end{mydef}
Similar notions have been previously considered \cite{kaplan2006chaotic, hunt1996maximum,feng2022dimension}.
The \textbf{Sharpness Dimension} SD measures the effective number of expanding directions of the dynamics on the attractor before global contraction dominates. It is defined from the ordered global sharpness indices $\lambda_1 \ge \lambda_2 \ge \dots \ge \lambda_d$, which quantify worst-case logarithmic stretching rates along principal directions on the attractor. Let $j^*$ be the largest index such that $\sum_{k=1}^{j^*} \lambda_k \ge 0$; then $j^*$ is the maximal dimension in which volumes do not contract. Indeed, we observe that, in the case of a proper set, SD is strictly smaller than the ambient dimension in the EoS regime where $\lambda_1$ is expected to be positive.

\vspace{3mm}
\paragraph{Generalization Bounds}
We are now ready to formulate our main theorem and establish the connection between our novel complexity measure and the worst-case generalization gap over the random pullback attractor. 
We start by stating our main assumptions, the first two of which are standard practice in the literature
\citep{simsekli2020hausdorff,andreeva2024topological,dupuis2024uniform, birdal2021intrinsic}:

\begin{assumption}[Boundedness of Loss]
\label{assum:bounded_loss}
We assume the loss function $\ell: \mathbb{R}^d \times \mathcal{Z} \to \mathbb{R}$ to be bounded. That is, there exists a constant $B > 0$ such that for all weights $w \in \mathbb{R}^d$ and $z \in \mathcal{Z} $ holds: $ 0 \leq \ell(w, \omega) \leq B$.
\end{assumption}

\begin{assumption}[Lipschitz Continuity of Loss]
\label{assum:lipschitz}
We assume that the loss function $\ell: \mathbb{R}^d \times \mathcal{Z} \to \mathbb{R}$ satisfies the following properties for all $z \in \mathcal{Z}$:
The function $w \mapsto \ell(w, z)$ is $L$-Lipschitz continuous for some $L > 0$. That is, for all $w_1, w_2 \in \mathbb{R}^d$: $|\ell(w_1, z) - \ell(w_2, z)| \leq L \|w_1 - w_2\|$.
\end{assumption}

\begin{assumption}[Regular Random Dynamics]
\label{assum:regular_dynamics}
For each dataset $S \in \mathcal{Z}^n$, let $(\Omega, \mathcal{F}, \mathbb{P}, \theta, \phi)$ be a $C^2$ discrete-time RDS according to Dfn.~\ref{def:rds}, with a unique compact random pullback attractor $\mathcal{A}_S(\omega)$. We assume:

\begin{enumerate}[leftmargin=*,topsep=0em,itemsep=0.01em]
     \item Non-Singularity: For $\mathbb{P}$-a.e. $\omega$ we assume  $\inf_{x \in \mathcal{A}(\omega)} \sigma_d(D \phi(1,\omega,x) > 0$

    \item \emph{Integrability}: $$\mathbb{E}\!\left[\sup_{x \in \mathcal{A}_S(\omega)} \ln \|D\phi_S(1,\omega,x)\|\right] < \infty\qquad \text{and} \qquad
    \mathbb{E}\!\left[\sup_{x \in \mathcal{A}_S(\omega)} \ln \|D^2\phi_S(1,\omega,x)\|\right] < \infty.$$
    \item \emph{Transition Index}: There exists an integer $j^* \in \{1, \dots, d-1\}$ such that: $$\sum_{i=1}^{j^*} \lambda_i \;\ge\; 0 
\;>\;
\sum_{i=1}^{j^*+1} \lambda_i,$$
    where $\lambda_i$ denotes the global sharpness of order $i$ (Dfn.~\ref{def:local_sharpness_unified}).
\item Bounded Distortion: For $A \in \mathbb{R}^{d \times d}$ and $j \in \{1, \dots, d\}$, define
\begin{equation}
    \|A\|_j := \sigma_1(A)\cdots\sigma_j(A),
\end{equation}
where $\sigma_1(A) \ge \cdots \ge \sigma_d(A)$ are the singular values of $A$. Equivalently, $\|A\|_j$ is the maximal expansion factor of $A$ on $j$-dimensional volumes.

We assume that the spatial variation of $\|D\phi(m,\omega,\cdot)\|_j$ over the attractor is subexponential in $m$: for each $j \in \{1, \dots, d\}$,
\begin{equation}
    \lim_{m \to \infty} \frac{1}{m}\, \mathbb{E}\!\left[\sup_{x \in \mathcal{A}(\omega)} \ln \|D\phi(m, \omega, x)\|_j \;-\; \inf_{x \in \mathcal{A}(\omega)} \ln \|D\phi(m, \omega, x)\|_j\right] = 0.
\end{equation}
\end{enumerate}
\end{assumption}
Intuitively, the $j$-volume growth rates along different orbits in $\mathcal{A}(\omega)$ may differ at sub-exponential order, but coincide to leading exponential order as $m \to \infty$. Conditions of this type, requiring the spatial variation of the cocycle to be subexponential, are commonly imposed in the smooth ergodic theory of random dynamical systems to obtain Lyapunov exponents that do not depend on the base point $x \in \mathcal{A}(\omega)$; see, e.g., Arnold~\cite{arnold2006random} for related formulations.
Under Assumption~\ref{assum:regular_dynamics}, we now present our main result.
\begin{mytheorem}{Generalization via Sharpness Dimension}{generalization_DLS}~

Let $S = \{z_1, \dots, z_n\} \sim \datadist$ be a dataset of size $n$. Let $(\Omega, \mathcal{F}, \mathbb{P}, \theta, \phi)$ be a discrete-time RDS according to Dfn~\ref{def:rds} such that Assump.~\ref{assum:regular_dynamics} holds.
Under Assumps.~\ref{assum:bounded_loss} and~\ref{assum:lipschitz}, there exists a constant $C>0$ s.t. with probability at least $1 - \zeta - \gamma$ over the joint draw $(S, \omega) \sim \datadist \otimes \mathbb{P}$, there exists $\delta_{n,\gamma} > 0$ such that for all $0 < \delta < \delta_{n,\gamma}$,
\begin{align*}
\mathcal{G}_S(\mathcal{A}(\omega))
&\leq 2 L \delta
+ 2 B \sqrt{\frac{4\, \dim_{\mathrm{S}}\mathcal{A}_S \> \log(1/\delta)}{n}} \\
&\quad + \frac{I_{\infty}(\mathcal{A}_S(\omega), S) + \log(1/\zeta)}{\sqrt{n}}
+ \frac{C B^2}{\sqrt{n}}.
\end{align*}
We recall that $\mathcal{G}_S(\mathcal{A}(\omega))$ denotes the worst-case generalization gap (see \eqref{eq:worst_case_gap}) and $I_{\infty}(\mathcal{A}_S(\omega), S)$ (see Dfn.~\ref{def:total_mutual_information})  denotes the total mutual information between the random pullback attractor $\mathcal{A}_S(\omega)$ and $S$. 
\end{mytheorem}

Thm~\ref{thm:generalization_DLS} rigorously links the generalization gap to the stability of training by showing that, at the EoS (where $\dimS < d$), the generalization gap is governed by the global geometry of the attractor rather than by any isolated solution. In particular, the observed expansion along at least one direction implies that the attractor is confined to a proper subset of the ambient parameter space, of dimension strictly smaller than $d$. The sharpness dimension $\dimS$ quantifies this effect by capturing the spectral balance between the expanding directions and the remaining contracting ones.

\vspace{-4mm}
\begin{proof}[Proof sketch]
    The main idea is to show that $\dimS$ upper bounds another notion of fractal dimension called the Minkowski dimension (MD, see Dfn.~\ref{def:minkowski_dimension}). As the MD is directly linked to covering numbers, we directly obtain a generalization bound by relying on existing results \citep{dupuis2024uniform}. Proving the fact that $\dimS$ is larger than MD is achieved by covering $\mathcal{A}_S(\omega)$ by using ellipsoids whose principal axes are determined by $\lambda_1, \dots, \lambda_d$, then computing the covering number for these ellipsoids. The proof is given in App.~\ref{proof:main}. 
\end{proof}
\vspace{-4mm}
Finally, the mutual term $I_\infty$, wich is used to handle the dependece of the dataset $S$ and the random pullback attractor $\mathcal{A}_S(\omega)$, can be avoided by using the recent `set stability' framework of \cite{tuci2025mutual}. Within this framework, we provide another generalization bound without the mutual information term in Thm.~\ref{thm:stability_bound_D_S}.

\vspace{-2mm}
\section{Empirical Results}
\vspace{-1mm}
\paragraph{Numerical Implementation}
Unlike trajectory-based approaches~\cite{birdal2021intrinsic,simsekli2020hausdorff,andreeva2024topological}, our generalization bound neither requires access to the full training trajectory nor necessitates tools of topological data analysis. Instead, our computational bottleneck lies in estimating the eigenvalue spectrum of the Hessian of a parameter near convergence (see App.~\ref{app:grok} for a formal approximation result). For small-scale networks, this can be achieved via (potentially randomized) GPU-parallelized singular value decomposition (SVD)~\cite{halko2011finding}. However, the quadratic memory complexity of the Hessian renders such approaches rapidly intractable as model size grows. This challenge is exacerbated by the fact that we require access not merely to the leading eigenvalues, but to the entire spectrum, which is dominated by a large mass of near-zero modes, precisely the regime where Krylov-based methods such as Lanczos iterations~\cite{lanczos1950iteration} become ineffective. To address this, we adopt \emph{stochastic Lanczos quadrature} (SLQ)~\cite{lin2016approximating,golub1969calculation}, a scalable spectral estimation technique that has recently been analyzed and successfully applied in the context of (moderately sized) neural networks~\cite{ghorbani2019investigation,papyan2018full}. %
Unlike classical SLQ where all runs probe a fixed matrix–vector product operator, we vary the operator across runs by computing Hessian–vector products on independently sampled minibatches. The operator is held fixed within each Lanczos run but resampled across runs so as to estimate the SD directly in expectation rather than post-hoc averaging over a single global spectrum.
We present a complexity analysis in App.~\ref{app:grok}.

\paragraph{Metrics}
We assess the \textbf{correlation} between various notions of complexity and generalization error by using Kendall's coefficients (KC) \cite{kendall_new_1938} as well as their ``granulated'' versions (GC)~\cite{jiang_fantastic_2019}. While the classical KC (denoted $\tau$) measures correlation between two quantities, it may fail to capture their causal relationship. Instead, one GC is defined for each hyperparameter (i.e., $\boldsymbol{\psi}_{\mathrm{LR}}$ for $\eta$ and $\boldsymbol{\psi}_{\mathrm{BS}}$ for $b$); it measures correlation when only this hyperparameter is varying.  %
Note that scaling the constant $B$ in Thm.~\ref{thm:generalization_DLS} does not affect the observed correlation between generalization and topological complexities.

As target metrics, we compute the \emph{generalization gap} (Gen. Gap) and \emph{loss gap} as the absolute difference in accuracy and loss, respectively, both between training and test data.

\paragraph{Notions of complexity} 
As a proxy to generalization, we use both Euclidean ($\|\cdot\|$) and data-dependent ($\rho$) persistent homology (PH) dimensions (PH-Dim)~\cite{birdal2021intrinsic,dupuis2023generalization}, and $\alpha$-weighted lifetime sums~\cite{andreeva2024topological} ($E_\alpha$) using a simulated trajectory of 5,000 iterations (PH-Dim-Fwd). We also consider the last 5,000 training iterations without trajectory simulation (PH-Dim-Bwd). We further include our \emph{RDS sharpness} as $\lambda_1$, the leading singular value (SV) of $I - \eta \nabla^2 \er(w)$ and our Sharpness Dimension (SD). In grokking settings, we also track \emph{Hessian trace} ($\mathrm{tr}(H)$) and the magnitude of its largest eigenvalue (\emph{sharpness}), obtained as the absolute value of the largest Hessian eigenvalue. When $\eta \cdot \mathrm{sharpness} > 1$, the two quantities correspond; otherwise, RDS sharpness is determined by a non-positive Hessian eigenvalue.

\vspace{-2mm}
\subsection{Analysis}
\vspace{-1mm}

\paragraph{Generalization in 3-layer small MLPs}
We begin by experimenting on classical learning scenarios with a small 3-layer MLP trained on MNIST that allows computation of the exact eigenvalue spectrum of Hessian matrices. We use a width of 16, ReLU activation, and no bias, for a total of 12,960 parameters. We train the networks using SGD, without momentum or weight decay (${WD}$), for 250 epochs using cross-entropy loss. We vary the learning rate (lr, $\eta$) and batch size ($b$) to define a $5 \times 5$ grid of hyperparameters. For each ($\eta$, $b$) pair, we estimate the expectation in our Sharpness Dimension $\dimS$ (SD) by sampling random minibatches, computing parameter Hessians, and taking the SVD of the $I - \eta \nabla^2 \er(w)$ matrices. RDS sharpness values $\lambda_k$ are then estimated as the average of the log SVs and SD is computed following Definition~\ref{def:local_sharpness_unified}. 

For our SLQ-based eigenvalue density estimator (SD-SLQ), we use 500 Lanczos runs of 100 steps each via random Rademacher initializations and full reorthogonalization. The expectation is estimated using a separate minibatch per Lanczos run, followed by Gaussian quadrature and kernel smoothing. We then compute SD by integrating the estimated spectral density and the eigenvalue-weighted spectral density using the trapezoidal rule.

\setlength{\columnsep}{6pt}

\begin{wrapfigure}{r}{0.52\columnwidth}
  \centering
    \vspace{-2mm}
    \includegraphics[width=\linewidth]{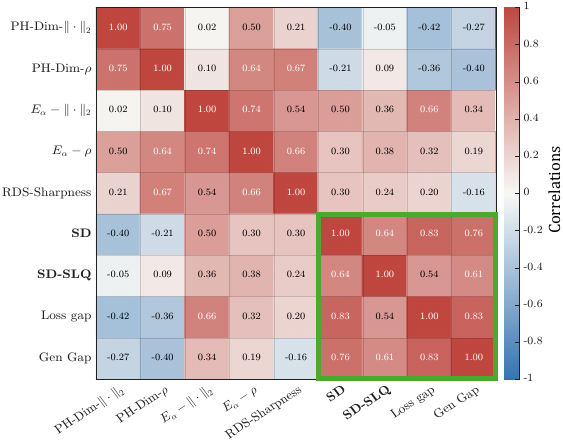}
    \vspace{-2mm}
    \caption{
    \small
    Correlations between various generalization indices and the empirical generalization gap on our small 3-layer MLP trained on MNIST dataset. The region indicated in green shows that our proposed Sharpness Dimension (SD) better predicts the generalization and loss gaps.
    }
    \label{fig:small-mlp-correlations1}
    \vspace{-2mm}
\end{wrapfigure}

Fig. \ref{fig:small-mlp-correlations1} shows correlations between various generalization measures and the empirical generalization and loss gaps. Both the Euclidean and data-dependent PH-Dims fail to capture the generalization behaviour. In addition, both Euclidean and data-dependent $E_\alpha$ show only weak positive correlation with the generalization gap. These results suggest that the trajectory-based topological indices may fail to accurately quantify generalization when networks operate at the edge of stability regime, where $I - \eta \nabla^2 \er(w)$ has singular values larger than 1, i.e., not contractive on average. The top RDS-Sharpness $\lambda_1$ alone also shows a weak negative correlation, suggesting that the largest SV alone is not indicative of generalization in this setting. In contrast, both SD and the approximate SD-SLQ show higher correlations with the gap. The full Hessian-spectrum SD achieves the strongest result indicating that SD better quantifies generalization in this edge of stability regime. 

\begin{wrapfigure}[19]{r}{0.44\columnwidth}
  \centering
  \vspace{-2.5mm}\includegraphics[width=0.40\columnwidth]{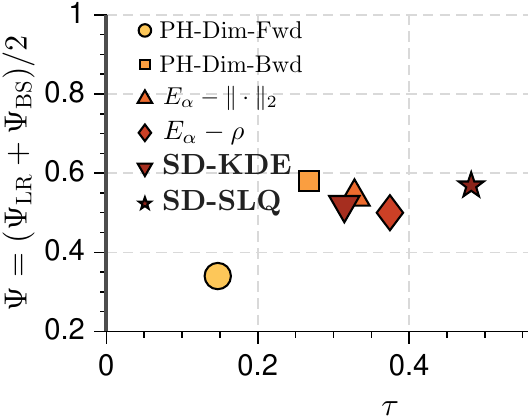}
  \caption{Kendall coefficients and their average granulated variant for the same 5-layer MLP trained using various batch sizes and learning rates. Our estimate further uses the SLQ approximation and its density estimating variant, SD-KDE.}
  \label{fig:large-mlp-kendall}
\end{wrapfigure}
\noindent\textbf{Generalization in larger 5-layer MLPs.}
To evaluate performance on larger networks, we consider a 5-layer MLP on MNIST with a width of 200, corresponding to 278,800 parameters. We follow the same training setup as the 3-layer MLP above and use the same $5 \times 5$ $(\eta, b)$ hyperparameter grid, with networks trained for 100 epochs.
Here, we use the same SLQ setup to estimate eigenvalue densities, with the number of Lanczos runs varying with minibatch size, corresponding to two training epochs. We compute both the SLQ-based histogram and the Gaussian kernel-smoothed estimate, denoted as SD-SLQ and SD-KDE, respectively. We cannot use the exact SVD, and hence do not report SD.

For various complexity measures, Fig.~\ref{fig:large-mlp-kendall} indicates that backward PH-Dim, forward $E_\alpha$'s, our SD-SLQ and SD-KDE achieve similar average GKC values, with backward PH-Dim performing best and SD-SLQ following closely. Yet, when comparing Kendall $\tau$ values, SD-SLQ outperforms all other indicators. Overall, when learning rates and batch sizes are grouped separately, the indicators perform similarly. However, when the full hyperparameter grid is considered, our dimension most accurately captures generalization. 

\begin{figure*}[t]
            \begin{center}
    \includegraphics[width=\linewidth]{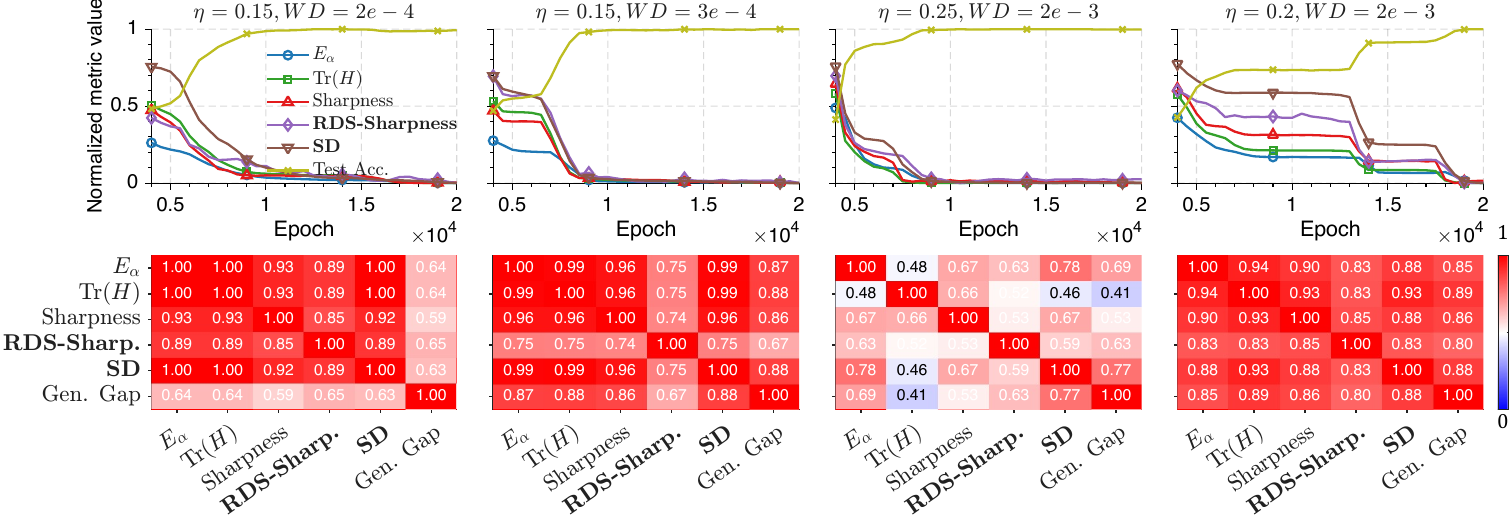}
            \end{center}
            \vspace{-1.5mm}
            \caption{\textbf{Grokking analysis} for different learning rates ($\eta$), weight decay ($WD$) and seeds across two architectures: (\textbf{top row}) 2-layer MLP. (\textbf{second row}) 2-layer MLP. Note that the \emph{suddenness } of the grokking behavior is best captured in the complexity measures we introduce: RDS-Sharpness and Sharpness Dimension (SD).\vspace{-5mm}}
            \label{fig:grokking}
\end{figure*}
\paragraph{Studying grokking}
Next, we study the phenomenon of \emph{grokking}, delayed and sudden generalization~\cite{power2022grokking, Nanda2023-hf,prieto2025grokking} through our EoS framework. %
We consider the task of \emph{arithmetic modulo 97}, a family of supervised learning tasks where two one-hot encoded inputs representing integers $a,b < p$ are used to predict the target $y=a*b \mod p$. $*$ is a binary operation and $p$ is a prime. In all our experiments, we use addition as the binary operation. %
The dataset size is defined as the percentage of the $97^2$ possible input pairs used for training, with the remainder used for testing as in~\cite{Nanda2023-hf} and \cite{power2022grokking}. We use a 40\%/60\% train/test split. 

To induce and observe grokking, we use a 2-layer MLP with 32 hidden features and GELU activation, trained using SGD with momentum and WD for 20k epochs. In all evaluations presented in Fig.~\ref{fig:grokking}, training acc.\ exceeds 90\% and reaches 100\% early, while test acc.\ continues to increase, a.k.a. \emph{grok}.
We report complexity measures for 100 uniformly spaced checkpoints as well as their intra-correlations and correlations with the Gen. Gap in Fig.~\ref{fig:grokking}. In App.~\ref{app:grok}, we also consider a 3-layer MLP with 32 hidden features and ReLU activation, trained with SGD using only WD.

Fig.~\ref{fig:grokking} shows that correlations our SD achieves with the generalization gap are comparable to or higher than those of $E_\alpha$, $\mathrm{tr}(H)$, and Sharpness, with our other measure, RDS sharpness, performing worse than SD for all the cases but the first. We observe that the Sharpness and $\mathrm{tr}(H)$ are less stable for certain hyperparameter choices, performing worse in one of the configurations. $E_\alpha$ shows stable correlations when using the training trajectories, whereas SD performs similarly or better even with only the model weights from the current epoch. %
These results indicate that SD captures information about training dynamics and the underlying attractor without requiring access to trajectory information like $E_\alpha$. In addition, during the grokking phase transition occurs~\cite{rubingrokking} simultaneously with changes in the structure of the Hessian spectrum.

\paragraph{Generalization in modern transformers: the case of GPT-2}
To assess whether the same behavior persists at a larger scale, we consider GPT-2~\cite{radford2019language} (124M parameters) trained from
scratch on WikiText-2~\cite{merity2016pointer}. We evaluate SGD and SGD with momentum on a $5 \times 5$ hyperparameter grid over learning rate
and effective batch size, and AdamW~\cite{loshchilov2019adamw} on a $4 \times 3 \times 3$ hyperparameter grid over learning rate, batch size, and
weight decay. We report the correlations between the resulting complexity measures and the empirical loss gap and 0--1
loss gap. 
For this experiment we additionally report SD-PS, an equal-mass pseudo-spectrum estimator computed from the smoothed SLQ density; see App.~\ref{app:sd-ps} for details.
Fig.~\ref{fig:GPT2} indicates that the qualitative conclusions remain consistent with our smaller-scale
experiments, with the strongest distinction appearing for AdamW. In particular, classical sharpness does not provide a
reliable indicator of generalization, and under AdamW, it shows a strong negative correlation with the loss gaps. By
contrast, the RDS-based quantities remain informative overall, with the SD variants, especially the smoothed estimates
SD-KDE and SD-PS, showing the most consistent positive correlations across optimizers. PH-Dim and $E_\alpha$ also
provide useful signals in some settings, but their behavior is less uniform. Overall, these results suggest that, in this
transformer setting as well, generalization is better captured by SD and related RDS-based quantities than by classical
sharpness alone.

\begin{figure*}[t]
            \begin{center}
    \includegraphics[width=\linewidth]{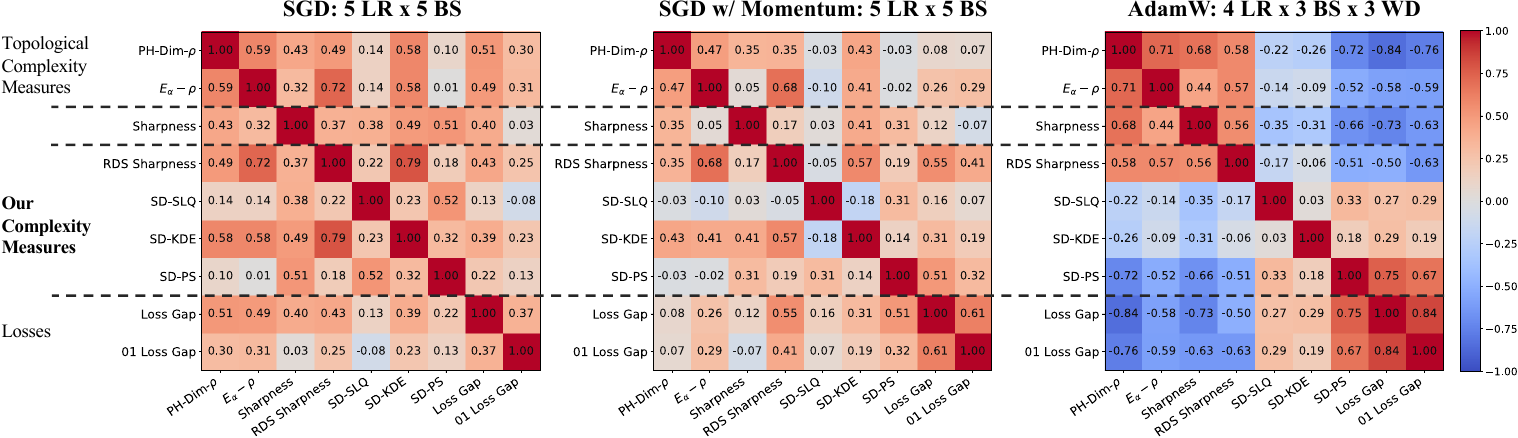}
            \end{center}
            \vspace{-2.5mm}
            \caption{\textbf{Correlation Matrices Corresponding to GPT-2 Trained on WikiText2-Dataset} for different learning rates (LR), batch sizes (BS) and weight decay values (WD), across three different optimizers: SGD, SGD with momentum and AdamW.\vspace{-4mm}}
            \label{fig:GPT2}
\end{figure*}

\section{Conclusion}
Training neural networks at the edge of numerical stability challenges classical generalization theories that assume convergence to a single solution. In this work, we show that this regime is instead governed by the \textbf{geometry of a random pullback attractor} induced by stochastic optimization viewed as a random dynamical system. Building on this, we introduce the Sharpness Dimension (SD), a Lyapunov-inspired spectral complexity measure that captures the effective dimensionality of expanding directions on the attractor. We prove a worst-case generalization bound over the entire attractor, establishing SD as a principled quantity linking chaotic optimization dynamics to generalization. Empirically, SD reliably predicts generalization gaps across optimization regimes, correlates with grokking phase transitions, and remains computable at scale via stochastic Lanczos quadrature.

\paragraph{Limitations \& future work}
Currently, we are required to estimate the full Hessian spectrum, which, despite scalable approximations, remains computationally demanding for very large models. 
Future work involves tightening the theory under weaker regularity assumptions, extending it to adaptive optimizers and large-scale architectures such as transformers~\cite{zhang2024transformers}.

\paragraph{Acknowledgments}
The authors are grateful for the support of the Excellence Strategy of local and state governments.
T. B. was supported by a UKRI Future Leaders Fellowship (MR/Y018818/1). The authors acknowledge support from the UK AI Research Resource (AIRR Isambard AI) through grant 0251-4584-0945-1 - TopoFound. U.\c{S}. is partially supported by the French government under the management of Agence Nationale de la Recherche as part of the ``Investissements d'avenir'' program, reference ANR-19-P3IA-0001 (PRAIRIE 3IA Institute). M.T and U.\c{S}. are partially supported by the European Research Council Starting Grant DYNASTY – 101039676.

\bibliographystyle{apalike}
\bibliography{arxiv}     %

\clearpage
\appendix

\appendix
\onecolumn

\section*{Appendix}
We provide here the technical details and full proofs omitted from the main paper, along with supplementary experiments that further support our findings. The appendix is structured as follows:

\begin{itemize}[topsep=0em,itemsep=0em,leftmargin=*]
    \item \textbf{Appendix A} introduces additional background material from random dynamical systems, random attractors, and worst-case generalization bounds.
    \item \textbf{Appendix B} presents the complete proofs of all theoretical results stated in the main text, as well as several supplementary theoretical developments.
    \item \textbf{Appendix C} describes the experimental setups and implementation details required to reproduce our empirical results.
    \item \textbf{Appendix D} We provide additional empirical results.
\end{itemize}
\vspace{-2mm}
\section{Theoretical Background}
\vspace{-2mm}
This section provides the theorems and technical background necessary for our main results.
\vspace{-2mm}
\subsection{Random Attractor}
\label{sec:random_attractor}
The notion of an \emph{attractor} is one of the fundamental concepts in dynamical systems theory. While this concept has been central in the study of deterministic systems for several decades, its extension to stochastic systems is more subtle.

In stochastic systems (like SGD), the "noise" never stops kicking the particle. Therefore, the attractor is likely not to be a single static point. Instead, it is a \emph{random set}, a moving target that fluctuates in shape and position over time, driven by the specific realization of the noise. The theory of attractors has been thoroughly developed over the past decades \citep{crauel1994attractors, arnold2006random}. Unlike the autonomous deterministic case, Random Dynamical Systems (RDS) are inherently non-autonomous and exhibit multiple notions of attraction, such as pullback and forward attraction. All definitions rely on the concept of a compact, invariant, random set.

\begin{definition}[\textbf{Compact and Invariant Random Sets}]
\label{def:invariant_random_set}
Let $\phi$ be an RDS on a Polish space $X$. A mapping $A: \Omega \to \mathcal{P}(X)$ is called a \textbf{compact random set} if:
\begin{enumerate}
    \item \textbf{Measurability:} The indicator function $\omega \mapsto \mathbf{1}_{A(\omega)}(x)$ is measurable for each fixed $x \in X$.
    \item \textbf{Compactness:} $A(\omega)$ is \textbf{compact} (i.e., closed and bounded) for $\mathbb{P}$-almost every $\omega \in \Omega$.
\end{enumerate}

The set $A$ is said to be an \textbf{invariant compact random set} if additionally it satisfies \textbf{$\phi$-Invariance}:
\[
\phi(t,\omega,A(\omega)) = A(\theta^t \omega), \quad \forall t \in \mathbb{T}^+, \quad \mathbb{P}\text{-a.s.}
\]
\end{definition}

\begin{definition}[\textbf{Modes of Random Attraction}]
\label{def:random_attractors}
Let $(\theta, \phi)$ be an RDS on a Polish space $X$, and let $A$ be an invariant compact random set according to Dfn.~\ref{def:invariant_random_set}. Let $\mathcal{S}$ be a collection of bounded subsets of $X$ (e.g., bounded initializations). We denote $\operatorname{dist}(E,F) := \sup_{x \in E} \inf_{y \in F} d(x,y)$. $A$ is called:

\begin{enumerate}
    \item \textbf{A Random Pullback Attractor} if for all $D \in \mathcal{S}$:
    \[
    \lim_{t \to \infty} \operatorname{dist} \left( \phi(t, \theta^{-t}\omega) D, A(\omega) \right) = 0 \quad \text{a.s.}
    \]
    \textbf{Intuition:} This asks: "If I started training infinitely long ago ($t \to -\infty$) with noise $\theta^{-t}\omega$, where would I be \emph{right now} ($t=0$)?"

    \item \textbf{A Random Forward Attractor} if for all $D \in \mathcal{S}$:
    \[
    \lim_{t \to \infty} \operatorname{dist} \left( \phi(t, \omega) D, A(\theta^t\omega) \right) = 0 \quad \text{a.s.}
    \]
    \textbf{Intuition:} This asks: "If I start training \emph{now} ($t=0$), where will I be in the distant future?"
    \textbf{SGD Context:} This describes the physical trajectory of a specific training run.

    \item \textbf{A Weak (Point) Attractor} if the convergence holds in probability rather than almost surely, (for singleton sets $D=\{y\}$).
\end{enumerate}
\end{definition}

Fig. \ref{fig:pullback_fractall_app} complements our definitions and results with an example. Indeed in the main write up we focus on the notion of random pullback attractor.

\begin{figure}[t]
\centering\includegraphics[width=\columnwidth]{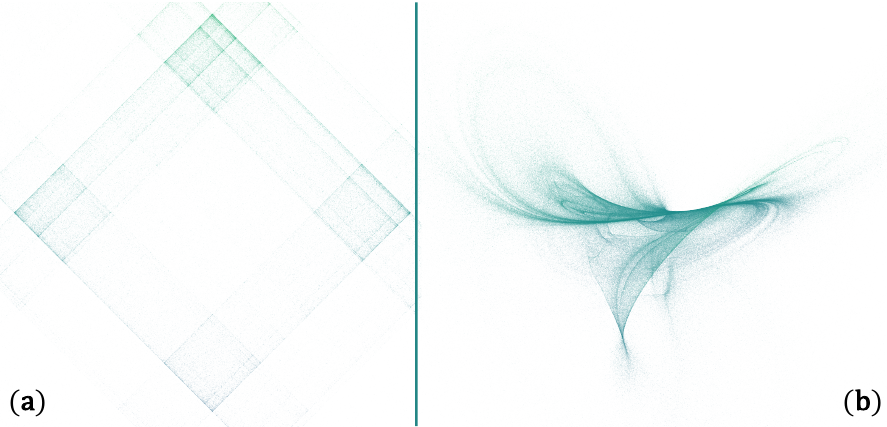}
\caption{\textbf{(a) Fractal Pullback Attractor at the Edge of Stability.} 
Visualization of the random snapshot attractor $\mathcal{A}(\zeta)$ generated by 
Stochastic Gradient Descent (SGD) on a non-linear function 
$L(x) = \frac{1}{2}\|x\|^2 - A \prod_{i=1}^2 \cos(k x_i)$ ($A=2.0, k=4.0$). 
The system is evolved for $T=250$ steps with a learning rate $\eta=0.15$ 
placing the dynamics at the edge of stability. The figure illustrates the 
collapse of $6 \times 10^5$ particles onto a fractal skeleton under a shared 
noise realization $\zeta = \{\xi_t\}_{t=1}^T$, where $\xi_t \sim \mathcal{N}(0, \sigma^2 I)$ 
with $\sigma=0.1$. The intricate filaments emerge from the recursive 
stretching and folding of the state space, characteristic of chaotic 
synchronization in random dynamical systems (RDS).
\textbf{(b) 3D chaotic attractor of an RDS.}  Illustration of a chaotic random pullback attractor using a stochastic generalized Nosé-Hoover system \cite{posch1986canonical}.} 

\label{fig:pullback_fractall_app}
\vspace{-5mm}
\end{figure}

\subsection{Fractal Dimensions}
We start by introducing the notion of a covering, which plays a central role throughout this paper and is especially important in the proofs of our main results.
\begin{definition}[Covering of a set]
\label{def:covering}
Let $X \subseteq \mathbb{R}^d$ be a set. A $\delta$-cover of $X$ is a family $(U_i)_i$ of sets with diameter at most $\delta$ such that
\[
X \subseteq \bigcup_i U_i.
\]
If $X$ is bounded, a $\delta$-cover of minimal cardinality is called a \emph{minimal cover}, and its cardinality is called the \emph{covering number}, denoted by $\mathcal{N}(X, \delta)$.
\end{definition}
Based on Dfn.~\ref{def:covering}, we are now able to define the upper Minkowski dimension, which will later be bounded using our new notion of fractal dimension.
\begin{definition}[Minkowski dimensions]
\label{def:minkowski_dimension}
Let $X \subset \mathbb{R}^d$ be a bounded set. The lower and upper Minkowski (or box) dimensions of $X$ are defined as
\[
\underline{\dim}_{\mathrm{B}}(X)
:=
\liminf_{\delta \to 0}
\frac{\log \mathcal{N}(X, \delta)}{\log (1/\delta)},
\qquad
\overline{\dim}_{\mathrm{B}}(X)
:=
\limsup_{\delta \to 0}
\frac{\log \mathcal{N}(X,\delta)}{\log (1/\delta)}.
\]
If the two quantities agree, the resulting limit is called the \emph{Minkowski dimension} of $X$, written $\dim_{\mathrm{B}}(X)$.

\end{definition}

\subsection{Data-dependent worst-case generalization bounds}
Before presenting the proofs of our main results, we state the following technical assumption, lemmata and preperation which are necessray to formulate our main result. 

Recent advances in topological generalization bounds rely on data-dependent worst-case generalization bounds, leveraging PAC-Bayesian theory on random sets \cite{dupuis2024uniform} or the stability-based framework of \citet{tuci2025mutual}.

\begin{theorem}[Egoroff's theorem]
\label{thm:egoroff}
Let $(\Omega,\mathcal{F},\mathbb{P})$ be a probability space, and let $f$ and $(f_n)_{n\in\mathbb{N}}$ be measurable functions on $\Omega$. Suppose that
\[
f_n(x) \longrightarrow f(x) \quad \text{for almost all } x \in \Omega.
\]
Then, for any $\varepsilon > 0$, there exists a measurable set $\Omega_\varepsilon \in \mathcal{F}$ such that
\[
\mathbb{P}(\Omega_\varepsilon) \ge 1 - \varepsilon,
\]
and the convergence of $f_n$ to $f$ is uniform on $\Omega_\varepsilon$.
\end{theorem}
This theorem allows us to relate the notion of covering (see Dfn.~\ref{def:covering}) to the Minkowski dimension (see Dfn.~\ref{def:minkowski_dimension}). Similar ideas have been used in previous work \citep{simsekli2020hausdorff,dupuis2023generalization,dupuis2024uniform}.
\begin{corollary}[Uniform control of covering numbers]
\label{cor:uniform_covering_minkowski}
Let $(\Omega_U,\mathcal{F}_U,\mathbb{P}_U)$ be the probability space supporting the random variable $U$, and let $(\mathcal{Z},\mathcal{F})$ denote the data space. Assume that the random set $\mathcal{W}_{S,U}$ is almost surely of finite diameter. Then, for any $\gamma>0$, there exists a measurable set $\Omega_\gamma \subseteq \mathcal{F}^{\otimes n} \otimes \mathcal{F}_U$ with
\[
\mu^{\otimes n} \otimes \mathbb{P}_U(\Omega_\gamma) \ge 1-\gamma,
\]
such that, on $\Omega_\gamma$, the following holds: there exists $\delta_{n,\gamma}>0$ for which, for all $0<\delta<\delta_{n,\gamma}$,
\[
\log \bigl| \mathcal{N}(\mathcal{W}_{S,U}, \delta) \bigr|
\;\le\;
2\, \upperbox(\wcal_{S,U}) \, \log\!\left(\tfrac{1}{\delta}\right).
\]
\end{corollary}
\begin{proof}
Let $(\Omega_U,\mathcal{F}_U,\mathbb{P}_U)$ denote the probability space supporting the auxiliary random variable $U$, and let $\mathcal{F}$ denote the $\sigma$-algebra associated with the data space $\mathcal{Z}$. Assume that the random set $\mathcal{W}_{S,U}$ is almost surely of finite diameter.

By definition of the upper Minkowski (box-counting) dimension, we have $\mu^{\otimes n} \otimes \mathbb{P}_U$-almost surely
\[
\dim_{\mathrm{B}}(\mathcal{W}_{S,U})
:=
\limsup_{\delta \to 0}
\frac{\log \bigl|\mathcal{N}(\mathcal{W}_{S,U}, \delta)\bigr|}{\log(1/\delta)}
=
\lim_{\delta \to 0} f_\delta(\mathcal{W}_{S,U}),
\]
where we define
\[
f_\delta(\mathcal{W}_{S,U})
:=
\sup_{0 < r < \delta}
\frac{\log \bigl|\mathcal{N}(\mathcal{W}_{S,U},r)\bigr|}{\log(1/r)}.
\]

Let $(\delta_k)_{k \ge 0}$ be a decreasing sequence of positive numbers in $(0,1)$ such that $\delta_k \to 0$, and fix $\gamma > 0$. By Egoroff’s theorem (see \citealp{bogachev2007measure}), there exists a measurable set
\[
\Omega_\gamma \in \mathcal{F}^{\otimes n} \otimes \mathcal{F}_U
\]
such that
\[
\mu^{\otimes n} \otimes \mathbb{P}_U(\Omega_\gamma) \ge 1 - \gamma,
\]
and such that the convergence
\[
f_{\delta_k}(\mathcal{W}_{S,U}) \longrightarrow \dim_{\mathrm{B}}(\mathcal{W}_{S,U})
\]
is uniform on $\Omega_\gamma$.

Consequently, there exists an index $k_{\gamma,n} \ge 0$ such that for all $k \ge k_{\gamma,n}$ and all $(S,U) \in \Omega_\gamma$,
\[
f_{\delta_k}(\mathcal{W}_{S,U})
\le
2\,\dim_{\mathrm{B}}(\mathcal{W}_{S,U}).
\]

Therefore, for any $0 < \delta < \delta_k$ and any $(S,U) \in \Omega_\gamma$, we obtain
\[
\log \bigl|\mathcal{N}(\mathcal{W}_{S,U}, \delta)\bigr|
\le
2\,\dim_{\mathrm{B}}(\mathcal{W}_{S,U}) \, \log\!\left(\tfrac{1}{\delta}\right),
\]
which concludes the proof.
\end{proof}

The framework of \citet{dupuis2024uniform} is based on information-theoretic quantities. In particular, we provide below a precise definition of the total mutual information term appearing in our main theoretical results; see \citep{van2014renyi,hodgkinson2022generalization} for further background.

\begin{definition}[Total mutual information]
\label{def:total_mutual_information}
Let $X$ and $Y$ be two random elements defined on a probability space $(\Omega,\mathcal{F},\mathbb{P})$ (note that the codomains of $X$ and $Y$ may be distinct). We define the total mutual information between $X$ and $Y$ by
\[
I_\infty(X,Y)
=
\log \sup_{A \in \mathcal{A}_{X,Y}}
\frac{\mathbb{P}_{X,Y}(A)}
{\mathbb{P}_X \otimes \mathbb{P}_Y(A)} .
\]
\end{definition}

\paragraph{Generalization Bounds via Mutual Information}
Many studies have employed information-theoretic techniques, particularly within the “fractal-based” literature \citep{simsekli2020hausdorff,birdal2021intrinsic}. A unifying perspective recently emerged with the PAC-Bayesian theory for random sets \citep{dupuis2024uniform}, which was recently employed by \citet{andreeva2024topological} to establish generalization bounds based on novel topological complexity measures. Informally, all these bounds are of the following form\footnote{We use the notation $\lesssim$ in informal statements where absolute constants have been omitted.}:
\begin{align}
\label{eq:informal-topological-bounds}
{\textstyle\sup_{w \in \wcal_{S,U}} } \big( \risk(w) - \er(w) \big) \lesssim \sqrt{\frac{\mathbf{C}(\wcal_{S,U}) + \mathrm{IT} + \log(1/\zeta)}{n}},
\end{align}
with probability at least $1 - \zeta$. The term $\mathrm{IT}$ is an \emph{information-theoretic} (IT) term, typically the \emph{total} mutual information between the dataset $S$ and the set $\wcal_{S,U}$
The aforementioned bounds differ in the choice of complexity measure $\mathbf{C}(\mathcal{W}_{S,U})$, but all include an IT term. 

By adapting Corollary 33 of \citet{dupuis2024uniform} to our framework, we derive the following Corollary 

\begin{corollary}
\label{cor:covering_bound}
Assume that $\ell(w,z)$ is $L$-Lipschitz in $w$, bounded, and that $\mathcal{W}_{S,U}$ is almost surely bounded. Then there exists a constant $C>0$ such that, for any $\lambda, \delta > 0$, with probability at least $1-\zeta$ under the joint law of (S,U), we have
\[
\sup_{w \in \mathcal{W}_{S,U}} \left(\risk(w)- \er(w) \right) \leq 2L\delta  + 2B  \sqrt{\frac{2 \log |\mathcal{N}(\mathcal{W}_{S,U},\delta)|}{n}} + \frac{I_{\infty}(\mathcal{W}_{S,U},S) + \log\frac{1}{\zeta}}{\lambda} + C\lambda \frac{B^2}{n}
\]
\end{corollary}
By combining Cor.~\ref{cor:uniform_covering_minkowski} and Cor.~\ref{cor:covering_bound} and applying a union bound, we can now state the following theorem, which serves as a key foundation for our subsequent analysis.

\begin{corollary}
\label{cor:uniform_generalization_bound}
Assume that the loss $\ell(w,z)$ is $L$-Lipschitz in $w$, bounded by $B$, 
and that the random set $\mathcal{W}_{S,U}$ is almost surely of finite 
diameter. Then there exists a constant $C>0$ such that, for any 
$\lambda > 0$, with probability at least $1-\zeta-\gamma$ under the 
joint law of $(S,U)$, there exists $\delta_{n,\gamma} > 0$ such that 
for all $0 < \delta < \delta_{n,\gamma}$,
\begin{multline}
\sup_{w \in \mathcal{W}_{S,U}} 
  \bigl( \mathcal{R}(w) - \widehat{\mathcal{R}}_S(w) \bigr)
\le 2 L \delta
+ 2 B \sqrt{\frac{4\, \overline{\dim}_{\mathrm{B}}(\mathcal{W}_{S,U}) 
  \, \log(1/\delta)}{n}} \\
+ \frac{I_{\infty}(\mathcal{W}_{S,U}, S) + \log(1/\zeta)}{\lambda}
+ \frac{C \lambda B^2}{n}.
\end{multline}
\end{corollary}
\begin{remark}
    The parameter $\delta_n$ appearing in \Cref{thm:fractal-dim-bound} 
    deserves a brief comment. As can be seen from the proof, it 
    quantifies the uniformity in $n$ of the limit defining the upper 
    box-counting dimension: if this convergence is uniform in $n$, then 
    $\delta$ becomes independent of $n$. A similar parameter already 
    arises in \citep{dupuis2023generalization}, and 
    \citep{dupuis2024uniform} shows that $\delta$ can indeed be taken 
    independent of $n$ under a suitable convergence assumption on the 
    distributions of the random sets. We refer the reader to these 
    works for further details.
\end{remark}

\paragraph{Generalization Bounds via Random Set Stability}
A different perspective was taken by \citet{tuci2025mutual}, who showed that various 'topological generalization' bounds can be recovered within a novel stability framework. In our case, we make use of both frameworks. Before stating our main bound in terms of the stability parameter, we recall the following assumption on the random set, denoted by $\wcal_{S,U}$, which, in our setting, is the random attractor.
\begin{definition}
    \label{def:selection}
    A data-dependent selection of $\wcal_{S,U}$ is a deterministic mapping
   $\omega: \mathrm{CL}(\Rd) \times \zcal^n \to \Rd$ such that $\omega(\wcal_{S,U}, S') \in \wcal_{S,U}$, almost surely. In particular, we assume the existence of a random variable $\omega_0(\wcal_{S,U}, S')$ such that, almost surely, $\omega_0(\wcal_{S,U}, S') \in \argmax_{w \in \wcal_{S,U}} G_{S'} (w)$.
\end{definition}

\begin{assumption}[Random set stability by \citet{tuci2025mutual}]
    \label{ass:random-trajectory-stability}
    $\mathcal{W}_{S,U}$ is $\beta_n$-random set stable if for any $J\in \mathbb{N}^\star$ and any data-dependent selection $\omega$ of $\wcal_{S,U}$, there exists a map $\omega' : \mathrm{CL}(\mathbb{R}^d) \times \mathbb{R}^d \to \mathbb{R}^d$ such that:
    \begin{itemize}
        \item For any $S$, $U$ and $w\in\Rd$, $\omega' (\wcal_{S,U}, w) \in \wcal_{S,U}$ .
        \item For all $z \in \mathcal{Z}$ and two datasets $S, S' \in \zcal^n$ differing by $J$ elements we have:
        \begin{align*}
           \mathbb{E}_{U} [| \ell(\omega(\mathcal{W}_{S,U}, S), z) - \ell(\omega'(\mathcal{W}_{S', U},\omega(\mathcal{W}_{S,U},S)), z ) |] \leq \beta_n J.
      \end{align*}

    \end{itemize}
\end{assumption}

As noted by \citet{tuci2025mutual}, in the absence of algorithmic randomness (that is, when $U$ is constant), Assump.~\ref{ass:random-trajectory-stability} reduces to a special case of the celebrated random set stability notion introduced by \citet{foster2019hypothesis}. Stability assumptions are typically formulated for neighboring datasets; here, we instead consider a variant in which datasets differ by $J$ elements. The two formulations are equivalent, and we adopt this version to streamline subsequent proofs and simplify notation.

\begin{theorem}[\citet{tuci2025mutual} Theorem 4.3.]
    \label{thm:fractal-dim-bound}
        Suppose that the loss $\ell$ is bounded by $B$, $L$-Lipschitz and  Assump.~\ref{ass:random-trajectory-stability} hold, and that $\wcal_{S,U}$ is a.s. of finite diameter. Without loss of generality, assume that $ \beta_n^{-2/3}$ is an integer divisor of $n$. There exists $\delta_n > 0$ such that for all $\delta < \delta_n$
       \begin{align*}
        \mathbb{E} \bigg[\sup_{w \in \mathcal{W}_{S,U}} \left(\mathcal{R}(w) - \widehat{\mathcal{R}}_S(w) \right) \bigg] \leq  2 \mathbb{E} \bigg[ \frac{B}{n} + \delta L + \beta_n^{1/3} \left( 1 + B \sqrt{4\upperbox(\wcal_{S,U})\log \frac1{\delta}} \right) \bigg],
\end{align*}
where $\upperbox(\wcal_{S,U})$ is the upper box-counting dimension (see Dfn.~\ref{def:minkowski_dimension})
\end{theorem}
Indeed, we present a simplified version of the theorem here. Since we 
assume that the loss $\ell$ is globally $L$-Lipschitz, this constant 
appears directly in the bound. It would in fact suffice to assume 
Lipschitz continuity only on the random set itself.

We will next provide the proofs omitted from the main paper.

\newpage
\section{Omitted Proofs}
Before we present our main theorem. We recall some definition and basic facts, which will be used in the proof and improve readability. 
\subsection{Preliminaries}
\subsection*{Important Sets}
\begin{definition}[Minkowski Sum]
 Let $U$ and $V$ be two non-empty subsets of a vector space (in this context, $\mathbb{R}^d$). The Minkowski sum, denoted by $U \oplus V$, is the set formed by adding every element of $U$ to every element of $V$:
\[
U \oplus V := \{u+v \mid u \in U, v \in V\}
\]   

In the context of the proof, the \textbf{Minkowski sum} provides a rigorous way to define the ``thickening'' of a transformed set to account for approximation errors. Let $U \subset \mathbb{R}^d$ be a compact set and let $B(0, \epsilon)$ be the closed ball of radius $\epsilon > 0$ centered at the origin. The Minkowski sum $U \oplus B(0, \epsilon)$ corresponds exactly to the \textbf{closed $\epsilon$-neighborhood} of $U$:
\begin{equation}
\label{eqn:def_setnhbd}
    U \oplus B(0, \epsilon) = \{ y \in \mathbb{R}^d \mid \exists u \in U, \|y - u\| \le \epsilon \}
\end{equation}
Equivalently, this can be expressed using the distance function $\text{dist}(y, U) = \inf_{u \in U} \|y - u\|$:
\begin{equation}
    U \oplus B(0, \epsilon) = \{ y \in \mathbb{R}^d \mid \text{dist}(y, U) \le \epsilon \}
\end{equation}

\end{definition}
\subsection*{Geometric Interpretation: Ellipsoids and Linear Images of Balls}

Let $L \in \mathbb{R}^{d \times d}$ be a real matrix of rank
\[
k := \mathrm{rank}(L) \le d.
\]
We study the geometry of the image of the unit ball
\[
B(0,1) := \{ \xi \in \mathbb{R}^d : \|\xi\|_2 \le 1 \}
\]
under the linear map $L$.

Define
\[
E := L(B(0,1)) = \{ L\xi \in \mathbb{R}^d : \|\xi\|_2 \le 1 \}.
\]
Then $E$ is an ellipsoid if $L$ is invertible, and a \emph{degenerate ellipsoid} (i.e.\ a flattened ellipsoid lying in a lower-dimensional subspace) if $\mathrm{rank}(L) < d$.

\paragraph{Canonical Form of the Ellipsoid}
By the Singular Value Decomposition, there exist orthonormal bases
\[
\{u_1,\dots,u_d\} \subset \mathbb{R}^d, \qquad \{v_1,\dots,v_d\} \subset \mathbb{R}^d
\]
and singular values
\[
\sigma_1 \ge \sigma_2 \ge \cdots \ge \sigma_k > 0, 
\qquad 
\sigma_{k+1} = \cdots = \sigma_d = 0
\]
such that
\[
L = \sum_{i=1}^k \sigma_i \, v_i u_i^T.
\]
Equivalently,
\[
L u_i = \sigma_i v_i \quad \text{for } i=1,\dots,k, 
\qquad 
L u_i = 0 \quad \text{for } i = k+1,\dots,d.
\]
The singular values satisfy
\[
\sigma_i = \sqrt{\lambda_i(L^T L)},
\]
where $\lambda_i(L^T L)$ are the eigenvalues of $L^T L$ ordered in decreasing order. Since $\{u_1,\dots u_d \}$ an orthonormal basis we rewrite $\xi \in \mathbb{R}^d$ with $\|\xi\| \leq 1$  as
\[
\xi = \sum_{i=1}^d a_i u_i, 
\qquad 
\sum_{i=1}^d a_i^2 \le 1.
\]
Then
\[
L\xi = \sum_{i=1}^k \sigma_i a_i v_i.
\]
Therefore for $t_i := \sigma_i a_i$,
\[
E 
= 
\left\{ 
\sum_{i=1}^k t_i v_i 
\;\Bigg|\; 
\sum_{i=1}^k \frac{t_i^2}{\sigma_i^2} \le 1 
\right\}.
\]
Geometrically, $E$ lies in the $k$-dimensional subspace
\[
\mathrm{Im}(L) = \mathrm{span}\{v_1,\dots,v_k\},
\]
its principal axes are the directions $v_1,\dots,v_k$, and the length of the $i$-th semi-axis is $\sigma_i$.

\paragraph{Image of a Ball of Radius $\rho$}
Let $\rho > 0$. Since
\[
B(0,\rho) = \rho \, B(0,1),
\]
by linearity we have
\[
L(B(0,\rho)) = \rho \, L(B(0,1)) = \rho \, E.
\]
Thus the scaled ellipsoid is
\[
\rho E 
= 
\left\{ 
\sum_{i=1}^k t_i v_i 
\;\Bigg|\; 
\sum_{i=1}^k \frac{t_i^2}{(\rho \sigma_i)^2} \le 1 
\right\},
\]
with principal semi-axis lengths
\[
\rho \sigma_1, \; \rho \sigma_2, \; \dots, \; \rho \sigma_k.
\]

\paragraph{Image of a Shifted Ball.}
For any $x \in \mathbb{R}^d$,
\[
B(x,\rho) = \{x\} \oplus B(0,\rho),
\]
and therefore
\[
L(B(x,\rho)) = \{Lx\} \oplus \rho E.
\]
\paragraph{Covering Estimates for Ellipsoid}
\begin{lemma}[Ellipsoid covering]
\label{lemma:ellipsoid_covering}
Let $E \subset \mathbb{R}^d$ be an ellipsoid with semi-axes
\[
\sigma_1 \ge \sigma_2 \ge \cdots \ge \sigma_d > 0.
\]
Let $\rho > 0$, and let $j \in \{0,1,\dots,d\}$ be such that
\[
\sigma_{j+1} \le \rho
\quad
(\text{with } \sigma_{d+1} := 0).
\]
Then the minimal number of Euclidean balls of radius $\rho$ needed to cover $E$ satisfies
\[
\mathcal{N}(E,\rho) \;\le\; 3^d \prod_{i=1}^j \frac{\sigma_i}{\rho}.
\]
\end{lemma}

\begin{proof}
Up to translation and rotation, we may assume that
\[
E = \left\{ x \in \mathbb{R}^d : \sum_{i=1}^d \frac{x_i^2}{\sigma_i^2} \le 1 \right\}.
\]
Then $E$ is contained in the axis-aligned box
\[
Q := \prod_{i=1}^d [-\sigma_i, \sigma_i].
\]
Hence
\[
\mathcal{N}(E,\rho) \le \mathcal{N}(Q,\rho).
\]

We cover $Q$ by a grid of cubes of side length $\rho$. Along the $i$-th coordinate direction, the interval $[-\sigma_i, \sigma_i]$ can be covered by at most
\[
\left\lceil \frac{2\sigma_i}{\rho} \right\rceil
\;\le\;
1 + \frac{2\sigma_i}{\rho}
\]
intervals of length $\rho$. Therefore,
\[
\mathcal{N}(Q,\rho) \le \prod_{i=1}^d \left( 1 + \frac{2\sigma_i}{\rho} \right).
\]

Now assume that $\sigma_{j+1} \le \rho$.

\medskip

\noindent
\textbf{Step 1: Small axes.}
For $i \ge j+1$, since $\sigma_i \le \rho$, we have
\[
1 + \frac{2\sigma_i}{\rho} \le 1 + 2 = 3.
\]

\medskip

\noindent
\textbf{Step 2: Large axes.}
For $i \le j$, since $\sigma_i \ge \rho$, we have
\[
1 + \frac{2\sigma_i}{\rho}
\;\le\;
\frac{3\sigma_i}{\rho}.
\]
Indeed, this is equivalent to $1 \le \sigma_i/\rho$, which holds.

\medskip

\noindent
\textbf{Step 3: Combine.}
Therefore,
\[
\mathcal{N}(E,\rho)
\;\le\;
\prod_{i=1}^j \frac{3\sigma_i}{\rho}
\;\cdot\;
\prod_{i=j+1}^d 3
\;=\;
3^d \prod_{i=1}^j \frac{\sigma_i}{\rho}.
\]

This completes the proof.
\end{proof}

\paragraph{Covering for Minkowski Sum}

\begin{lemma}[Covering of Minkowski sums]
\label{lemma: minkowski_covering}
Let $X,Y \subset \mathbb{R}^d$ be bounded sets and let $\varepsilon, \delta > 0$. Assume that
\[
X \subset \bigcup_{i=1}^N B(x_i, \varepsilon),
\qquad
Y \subset B(0, \delta),
\]
where $N \in \mathbb{N}$ and $x_i \in \mathbb{R}^d$ for all $ i\in \{1,\dots,N\}$.
Then the Minkowski sum $A \oplus B$ satisfies
\[
X \oplus Y \subset \bigcup_{i=1}^N B(x_i, \varepsilon + \delta).
\]
In particular, the covering numbers satisfy
\[
\mathcal{N}(X \oplus Y, \varepsilon + \delta) \le \mathcal{N}(X, \varepsilon).
\]
\end{lemma}
\begin{proof}
Let $a \in X \oplus Y$. Then $a = x + y$ with $x \in X$ and $y \in Y$. By assumption, there exists $i \in \{1,\dots,N\}$ such that $x \in B(x_i, \varepsilon)$, hence $\|x - x_i\| \le \varepsilon$. Moreover, since $y \in B(0, \delta)$, we have $\|y\| \le \delta$. Therefore,
\[
\|a - x_i\| = \|x+y-x_i\| \le \|x - x_i\| + \|y\| \le \varepsilon + \delta,
\]
which shows that $a \in B(x_i, \varepsilon + \delta)$. This proves the claim.
\end{proof}

We will now prove the main result of our work. We begin by bounding the Minkowski dimension (see Dfn.~\ref{def:minkowski_dimension})
\begin{theorem}[Minkowski Dimension Bound]
\label{thm:KY_bound_general}
Let $(\Omega, \mathcal{F}, \mathbb{P}, \theta, \phi)$ be a discrete-time
random dynamical system according to Dfn.~\ref{def:rds}, such that for $\mathbb{P}$-a.e. $\omega$, the map $x \mapsto \phi(1, \omega, x)$ is $C^2$. 
Suppose the following hold:

\begin{enumerate}
    \item \textbf{Non-Singularity:} For $\mathbb{P}$-a.e. $\omega$ we assume 
    \begin{align}
        \inf_{x \in \mathcal{A}(\omega)} \sigma_d(D \phi(1,\omega,x) > 0
    \end{align}
    
    \item \textbf{Random Invariant Set:} There exists a random compact set $\mathcal{A} = \{ \mathcal{A}(\omega) \}_{\omega \in \Omega}$ in $\mathbb{R}^d$ that is invariant under the cocycle $\phi$, i.e.,
    \begin{equation}
    \phi(n, \omega, \mathcal{A}(\omega)) = \mathcal{A}(\theta^n \omega), \quad \mathbb{P}\text{-a.s. for all } n \in \mathbb{N}.
    \end{equation}
    The mapping $\omega \mapsto \mathcal{A}(\omega)$ is measurable in the sense that the distance function $\omega \mapsto \text{dist}(x, \mathcal{A}(\omega))$ is a random variable for every $x \in \mathbb{R}^d$.

  \item \textbf{Integrability:} The logarithms of the linearized growth and the local curvature are integrable over the attractor. Specifically, we assume:
    \begin{equation}
        \mathbb{E} \left[ \sup_{x \in \mathcal{A}(\omega)} \ln\|D\phi(1, \omega, x)\| \right] < \infty
    \end{equation}
    and
    \begin{equation}
        \mathbb{E} \left[ \ln \sup_{x \in \mathcal{A}(\omega)} \| D^2 \phi(1, \omega, x) \| \right] < \infty.
    \end{equation}
  
    \item \textbf{Transition Index:} There exists an integer $j^* \in \{1, \dots, d-1\}$ such that:
    \[
    \sum_{i=1}^{j*} \lambda_i \geq 0 \quad \text{and} \quad \sum_{i=1}^{j^*+1} \lambda_i < 0.
    \]
  .
\item \textbf{Bounded Distortion}: For $A \in \mathbb{R}^{d \times d}$ and $j \in \{1, \dots, d\}$, define
\begin{equation}
    \|A\|_j := \sigma_1(A)\cdots\sigma_j(A),
\end{equation}
where $\sigma_1(A) \ge \cdots \ge \sigma_d(A)$ are the singular values of $A$. Equivalently, $\|A\|_j$ is the maximal expansion factor of $A$ on $j$-dimensional volumes.

We assume that the spatial variation of $\|D\phi(m,\omega,\cdot)\|_j$ over the attractor is subexponential in $m$: for each $j \in \{1, \dots, d\}$,
\begin{equation}
    \lim_{m \to \infty} \frac{1}{m}\, \mathbb{E}\!\left[\sup_{x \in \mathcal{A}(\omega)} \ln \|D\phi(m, \omega, x)\|_j \;-\; \inf_{x \in \mathcal{A}(\omega)} \ln \|D\phi(m, \omega, x)\|_j\right] = 0.
\end{equation}
    In other words, the maximal and minimal $j$-volume growth rates over $\mathcal{A}(\omega)$ agree at exponential scale.
\end{enumerate}

Define $\lambda_1 \geq \lambda_2 \geq \dots \geq \lambda_d$ be the one-step  exponents associated with the maximal expansion on the set $\mathcal{A}(\omega)$, defined as:
    \begin{align}
        \lambda_k := \mathbb{E} \left[ \sup_{w \in \mathcal{A}(\omega)} \ln \sigma_k(D\phi(1,\omega,w)) \right],
    \end{align}
    where $\sigma_k(A)$ denotes the $k$-th singular value of a linear map $A$. Note that by the integrability assumption, these values are finite. 

Then, for $\mathbb{P}$-almost every $\omega \in \Omega$, the \emph{upper Minkowski dimension} of the set $\mathcal{A}(\omega)$ is bounded by the  Sharpness Dimension $\dim_{\mathrm{S}}\mathcal{A}$:
\[
\overline{\dim}_M(\mathcal{A}(\omega)) \leq \dim_{\mathrm{S}}\mathcal{A}.
\]
\end{theorem}
\begin{proof}

The main idea behind the proof is the fact that the sets $\mathcal{A}(\omega)$ and $\mathcal{A}(\theta^K\omega)$ have the same distribution for any $K$. Hence, we will estimate the covering number of $\mathcal{A}(\theta^K\omega)$ which will enable us to link the covering number to the singular values of $D\phi$.  

\paragraph{Step 1: General Covering}
Let $(\Omega, \mathcal{F}, \mathbb{P},\phi)$ be a  $C^2$ random dynamical system. We begin by covering the set $\mathcal{A}(\omega)$, which is almost surely compact and, therefore, bounded. Hence, for any $R > 0$ we have a finite integer $N_1(\omega,R)$ such that 
\begin{align}
    \mathcal{A}(\omega) \subset \bigcup_{i=1}^{N_1(\omega,R)} B(x_i(\omega),R),
    \label{eqn:proof_main_1}
\end{align}
where $\{x_1(\omega), \dots, x_{N_1}(\omega)\}$ denote the centers of a finite covering of $\mathcal{A}(\omega)$ by balls of radius $R$. These centers can be chosen in a measurable way \cite{molchanov_theory_2017}. We denote the corresponding covering number $N_1(\omega,R)$ simply by $N_1$, or, to make the dependence on $\omega$ and $R$ explicit, by $N_1(\omega,R)$.

Since by assumption $\mathcal{A}(\omega)$ is $\phi$-invariant, we have 
$\mathcal{A}(\theta\omega) = \phi(1, \omega, \mathcal{A}(\omega))$ and thus by \eqref{eqn:proof_main_1}
\begin{equation}
\label{eqn:proof_main_2}
    \mathcal{A}(\theta\omega) \subset \bigcup_{i=1}^{N_1(\omega,R)} \phi(1, \omega, B(x_i, R)).
\end{equation}

\paragraph{Step 2: Local Approximation}
In the second step, we approximate $\phi$ via a second-order Taylor approximation.

By assumption, for every $\omega \in \Omega$ the map $x \mapsto \phi(1,\omega,x)$ is twice continuously differentiable. We define the random variable
\begin{align}
    C(\omega,R) := \frac{1}{2} \sup_{x \in \mathcal{A}(\omega)_R} \bigl\| D^2 \phi(1,\omega,x) \bigr\|,
    \label{eqn:thm_main_9}
\end{align}
where $\mathcal{A}(\omega)_R$ denotes the closed $R$-neighborhood of $\mathcal{A}(\omega)$, cf.\ \eqref{eqn:def_setnhbd}.

Since $\mathcal{A}(\omega)$ is bounded almost surely, its closed $R$-neighborhood $\mathcal{A}(\omega)_R$ is also compact almost surely. Because $D^2 \phi(1,\omega,\cdot)$ is continuous, the supremum in the definition of $C(\omega,R)$ is finite for almost every $\omega$. Hence,
\[
C(\omega,R) < \infty \quad \text{for almost every } \omega \in \Omega.
\]

We now study the image of the random ball $B(x_i(\omega), R)$ under the map $\phi(1,\omega,\cdot)$. As argued above, the random constant $C(\omega)$ is finite only almost surely and not uniformly in $\omega$. Therefore, we fix a set $\Omega_0 \subset \Omega$ with full measure, i.e.
\[
\mathbb{P}(\Omega_0) = 1,
\]
such that the constant $C(\omega,R)$ is finite for all $\omega \in \Omega_0$. In what follows, all arguments are carried out pathwise for $\omega \in \Omega_0$. Fix $i \in \{1,\dots,N_1 \}$ and we investigate the image of $\phi(1,\omega,B(x_i(\omega),R))$. 

Fix $\omega \in \Omega_0$. Let $i := i(\omega) \in \{1,\dots,N_1(\omega)\}$ (for notational simplicity we suppress the dependence on $\omega$ in the index). We consider the image of the ball $B(x_i(\omega), R)$ under $\phi(1,\omega,\cdot)$. Let $y(\omega) \in B(x_i(\omega), R)$ be arbitrary. Then there exists $\xi(\omega)$ with $\|\xi(\omega)\| \le R$ such that
\[
y(\omega) = x_i(\omega) + \xi(\omega).
\]
Since $\phi(1,\omega,\cdot)$ is $C^2$,  Taylor's theorem yields
\begin{align*}
    \phi(1,\omega, x_i(\omega) + \xi(\omega))
=
\phi(1,\omega, x_i(\omega))
+ D\phi(1,\omega, x_i(\omega)) \, \xi(\omega)
+ R(\omega, x_i(\omega), \xi(\omega)),
\end{align*}
where the remainder term satisfies the uniform bound
\[
\|R(\omega, x_i(\omega), \xi)\|
\le
C(\omega,R)\, \|\xi\|^2
\le
C(\omega,R)\, R^2.
\]

Since the point $y(\omega) \in B(x_i(\omega), R)$ was arbitrary, the above estimate holds uniformly for all $y(\omega) \in B(x_i(\omega),R)$. Therefore, the image of the ball $B(x_i(\omega), R)$ under $\phi(1,\omega,\cdot)$ satisfies the set inclusion
\begin{align}
        \label{eqn:proof_main_3}
        \phi(1,\omega, B(x_i(\omega), R))
\subset
\{\phi(1,\omega, x_i(\omega)) \}
\;\oplus\;
D\phi(1,\omega, x_i(\omega))\, B(0,R)
\;\oplus\;
B\bigl(0, C(\omega,R)\, R^2\bigr),
\end{align}
where $\oplus$ denotes the Minkowski sum of sets. Since $i$ was arbitrary, this holds for each $i \in \{1,\dots N_1 \}$.
In other words, we showed the set inclusion for almost all $\omega$.

\paragraph{Step 3: Covering Numbers}
We recall we have now by \eqref{eqn:proof_main_2} and \eqref{eqn:proof_main_3}
\begin{align}
\label{eqn:thm_proof_11}
    \mathcal{A}(\theta \omega) \subset \bigcup_{i=1}^{N_1} \{\phi(1,\omega, x_i(\omega)) \}
\;\oplus\;
D\phi(1,\omega, x_i(\omega))\, B(0,R)
\;\oplus\;
B\bigl(0, C(\omega,R)\, R^2\bigr) .
\end{align}
Now, we will obtain an esimate on the covering number of $\mathcal{A}(\theta\omega)$ by using the fact that each element in the union 
$$\{\phi(1,\omega, x_i(\omega)) \}
\;\oplus\;
D\phi(1,\omega, x_i(\omega))\, B(0,R)
\;\oplus\;
B\bigl(0, C(\omega,R)\, R^2\bigr)$$
is a dilated ellipsoid, so that we can use our result for estimating covering numbers for ellipsoids. 

More precisely, we have that for any $\rho > 0$
\begin{align}
\label{eqn:main_proof_5}
    \mathcal{N}(\mathcal{A}(\theta\omega), \rho) \leq \sum_{i=1}^{N_1} \mathcal{N} \left( \{ \phi(1,\omega, x_i(\omega)) \} \oplus D\phi(1,\omega, x_i(\omega)) B(0,R) \oplus B(0, C(\omega,R) R^2), \rho \right)
\end{align}

We now apply Lemma \ref{lemma: minkowski_covering} with
\[
X := D\phi(1,\omega, x_i(\omega))\, B(0,R),
\qquad
Y := B\bigl(0, C(\omega,R)\, R^2\bigr).
\]
The translation by $\phi(1,\omega, x_i(\omega))$ does not affect the covering number and can be ignored in the following estimates. 

 We further define as covering radius 
 $$\rho_1(\omega) := R\sup_{w \in \mathcal{A}(\omega)} \sigma_{j^*+1}(D\phi(1,\omega,w)).$$ 
 By Lemma \ref{lemma: minkowski_covering} we obtain 
\begin{align}
\label{eqn:main_proof4}
    \mathcal{N}(X \oplus Y, \rho_1(\omega) +C(\omega,R)R^2) \leq \mathcal{N}(X, \rho_1(\omega)) .
\end{align}
We recall that $j^*$ is the transition index and $\sigma_{i}(D\phi(1,\omega,x))$ denotes the $i$-th singular value of $D\phi(1,\omega,x)$ for $x \in \mathbb{R}^d$. 
Now we recall that X is an ellipsoid with semi axis lengths $$R\sigma_1(D\phi(1,\omega,x_i(\omega))), \dots, R\sigma_d(D\phi(1,\omega,x_i(\omega)))$$
since 
$$
X = D\phi(1,\omega, x_i(\omega))\, B(0,R) = R \cdot D\phi(1,\omega, x_i(\omega))\, B(0,1).
$$
Therefore, we are now ready to apply Lemma \ref{lemma:ellipsoid_covering}.  We note that 
$$
R\sigma_{j^*+1}(D\phi(1,\omega,x_i(\omega))) \leq R\sup_{w\in \mathcal{A}(\omega)}\sigma_{j^*+1}(D\phi(1,\omega,w) = \rho_1(\omega).$$ 
Hence, by Lemma \ref{lemma:ellipsoid_covering} we obtain 
\begin{align}
    \mathcal{N}(X,\rho_1(\omega)) \leq 3^d\prod_{k=1}^{j^*} \frac{R\sigma_k(D\phi(1,\omega,x_i(\omega)))}{\rho_1(\omega)} \leq 3^d\prod_{k=1}^{j^*} \frac{R\sup_{x \in \mathcal{A}(\omega)}\sigma_k(D\phi(1,\omega, x)}{\rho_1(\omega)} =: 3^d \mu(\omega).
    \label{eqn:thm_main_proof_8}
\end{align}
Notice that even though $X$ depends on the particular center $x_i(\omega)$, $\mu(\omega)$ is uniform over all the centers and does not depend on $i$. 

Now, let us define 
$$\rho(\omega) := \rho_1(\omega) + C(\omega,R)R^2$$ 
and observe that by combining \eqref{eqn:main_proof_5} and \eqref{eqn:main_proof4}, we have that 
\begin{align}
    \nonumber \mathcal{N}(\mathcal{A}(\theta\omega), \rho(\omega)) \leq& \sum_{i=1}^{N_1} \mathcal{N} \left( \{ \phi(1,\omega, x_i(\omega)) \} \oplus D\phi(1,\omega, x_i(\omega)) B(0,R) \oplus B(0, C(\omega,R) R^2), \rho(\omega) \right) \\
    \nonumber =& \sum_{i=1}^{N_1} \mathcal{N} \left(  D\phi(1,\omega, x_i(\omega)) B(0,R) \oplus B(0, C(\omega,R) R^2), \rho(\omega) \right) \\
    \leq& \sum_{i=1}^{N_1} \mathcal{N}(D\phi(1,\omega, x_i(\omega)) B(0,R) , \rho_1(\omega) 
    \label{eqn:thm_main_proof_6} \\
    \leq&  N_1(\omega,R) 3^d \mu(\omega),
    \label{eqn:thm_main_proof_7}
\end{align}
where \eqref{eqn:thm_main_proof_6} follows from \eqref{eqn:main_proof4} and \eqref{eqn:thm_main_proof_7} follows from \eqref{eqn:thm_main_proof_8}.

Now, we choose $R$ sufficiently small such that 
\begin{align}
    C(\omega,R)R^2 \leq \rho_1(\omega).
    \label{eqn:thm_main_10}
\end{align}
This is possible since when $R$ tends to $0$, $C(\omega,R)$ converges to a constant by definition (cf.\ \eqref{eqn:thm_main_9}) hence $C(\omega,R)R^2$ tends to $0$ as $R$ tends to $0$. On the other hand, since by assumption $\inf_{x \in \mathcal{A}(\omega)} \sigma_d(D\phi(1,\omega,x)) >0$, that makes $\rho_1(\omega) >0$, which yields \eqref{eqn:thm_main_10}. 

Therefore, for sufficiently small $R$, we have that
\begin{equation*}
    \rho(\omega) \leq 2\rho_1(\omega) =: \epsilon_1(\omega).
\end{equation*}

By the monotonicity of covering numbers, a larger radius requires fewer balls. 
Therefore, we have the following chain of inequalities:
\begin{align}
    \mathcal{N}(\mathcal{A}(\theta\omega), \epsilon_1(\omega)) \leq \mathcal{N}(\mathcal{A}(\theta\omega), \rho(\omega)) \leq N_1(\omega,R) 3^d\mu(\omega),
    \label{eqn:thm_main_12}
\end{align}
where the last step follows from \eqref{eqn:thm_main_proof_7}. 

At this stage we have estimated the covering number of $\mathcal{A}(\theta^1\omega)$. In the next step, we will iterate this idea to get an estimate on the covering number of $\mathcal{A}(\theta^K\omega)$.

\paragraph{Step 4: Global Iteration and Recursive Covering}
We now extend the one-step covering estimate to multiple time steps by induction over the discrete time index $K$. Fix a sufficiently small initial radius $R>0$ and suppose that $\mathcal{A}(\omega)$ is covered by $N_1(\omega,R)$ balls of radius $R$. By the invariance of the random attractor,
\[
\mathcal{A}(\theta^{K+1}\omega) = \phi(1, \theta^K\omega, \mathcal{A}(\theta^K\omega)) \quad \text{for all } K \in \mathbb{N}.
\]
Define the sequence of radii $\{\epsilon_K\}_{K \ge 0}$ by
\begin{equation}
\epsilon_{K+1}(\omega)
:=
2 \cdot \epsilon_K(\omega)
\sup_{x \in \mathcal{A}(\theta^K\omega)}
\sigma_{j^*+1}(D\phi(1, \theta^K\omega, x)),
\qquad
\epsilon_0 := R.
\end{equation}
Moreover, the one-step covering argument derived above applies to any sufficiently small covering radius. Consequently, if $\mathcal{A}(\theta^K\omega)$ is covered by balls of radius $\epsilon_K(\omega)$, then applying the same argument with $\omega$ replaced by $\theta^K\omega$ yields a cover of $\mathcal{A}(\theta^{K+1}\omega)$ by balls of radius $\epsilon_{K+1}(\omega)$. In other words, all previous estimates remain valid after the replacements $\mathcal{A}(\omega) \leftarrow \mathcal{A}(\theta^K\omega)$ and $\mathcal{A}(\theta\omega) \leftarrow \mathcal{A}(\theta^{K+1}\omega)$.

By the one-step covering estimate from \eqref{eqn:thm_main_12}, each ball of radius $\epsilon_K(\omega)$ is mapped into a set that can be covered by at most $\mu(\theta^K\omega)$ balls of radius $\epsilon_{K+1}(\omega)$. Consequently,
\begin{equation}
\label{eq:recursion}
\mathcal{N}(\mathcal{A}(\theta^K\omega), \epsilon_K(\omega))
\;\le\;
N_1(\omega,R)
\prod_{k=0}^{K-1} 3^{d} \cdot \mu(\theta^k\omega).
\end{equation}
Iterating the radius recursion yields the explicit expression
\begin{equation}
\label{eqn:thm_main_13}
\epsilon_K(\omega)
=
R\prod_{k=0}^{K-1}
\left(
2 \cdot \sup_{x \in \mathcal{A}(\theta^k\omega)}
\sigma_{j^*+1}(D\phi(1, \theta^k\omega, x))
\right).
\end{equation}

\paragraph{Step 5: Ergodic Limits and Dimension Bound}
To determine the asymptotic growth rate of the covering number, we begin by recalling the definition of the $i$-th sharpness exponent:
\begin{equation}
\lambda_i := \mathbb{E} \left[ \sup_{x \in \mathcal{A}(\omega)} \ln \sigma_i(D\phi(1,\omega,x)) \right].
\end{equation}

Taking logarithms in the covering estimate from \eqref{eq:recursion} and normalizing by $K$, we obtain
\begin{equation}
\label{eq:covering_number_bound}
\frac{1}{K} \ln \mathcal{N}(\mathcal{A}(\theta^K\omega), \epsilon_K(\omega)) \le \frac{1}{K} \ln N_1(\omega,R) + \frac{1}{K} \sum_{k=0}^{K-1}\ln \mu(\theta^k\omega) + d\ln 3.
\end{equation}

Since $N_1(\omega,R)$ is finite almost surely and independent of $K$, we have $\frac{1}{K} \ln N_1(\omega,R) \to 0$ as $K \to \infty$. By the Birkhoff Ergodic Theorem (see \citet[Theorem~3.10]{yunis2017birkhoff}) applied to the ergodic process $\omega \mapsto  \ln \mu(\omega)$\footnote{By our integrability assumption $\mu$ is also integrable, therefore the theorem can be applied.}, we obtain
\begin{align}
    \label{eq:main_part_1}
    \limsup_{K\to\infty} \frac{1}{n} \ln \mathcal{N}(\mathcal{A}(\theta^K\omega), \epsilon_K(\omega)) 
    \le \mathbb{E} \left[ \ln \mu(\omega) \right]  + d\ln 3  \\
    = \mathbb{E} \left[   \sum_{i=1}^{j^*} \sup_{x \in \mathcal{A}(\omega)} \ln \sigma_i(D\phi(1,\omega,x))  -j^* \sup_{x \in \mathcal{A}(\omega)}\ln \sigma_{j^*+1}(D\phi(1,\omega,x))  \right] + d\ln3.
\end{align}

\paragraph{Step 5: Asymptotic Scale Decay}

By using \eqref{eqn:thm_main_13}, taking logarithms and normalizing by $K$, we compute
\begin{align}
    \frac{1}{K} \ln \epsilon_K(\omega) 
    &= \frac{1}{K} \sum_{k=0}^{K-1} \left(\sup_{x \in \mathcal{A}(\theta^k\omega)} \ln \sigma_{j^*+1}(D\phi(1, \theta^k\omega , x))\right) + \ln2 + \frac{1}{K}\ln R.
\end{align}

Applying the Birkhoff Ergodic Theorem once more, we obtain
\begin{equation}
\label{eq:radii_bound}
    \lim_{K\to\infty} \frac{1}{K} \ln \epsilon_K(\omega) = \mathbb{E} \left[ \sup_{x \in \mathcal{A}(\omega)} \ln \sigma_{j^*+1}(D\phi(1, \omega, x)) \right] +  \ln2  = \lambda_{j^*+1} + \ln 2.
\end{equation} 

Our goal is to study the asymptotic growth of the covering numbers $\mathcal{N}(\mathcal{A}(\theta^K\omega), \epsilon_K(\omega))$ by dividing their logarithm by $K$ and passing to the limit $K \to \infty$. However, the current bound involve several scale-independent multiplicative constants (such as $3^d$) that obscure the leading exponential behavior and would not disappear when taking this limit. To resolve this, we refine the analysis by rewriting the recursive scales $\epsilon_K(\omega)$ in a form that makes their multiplicative structure explicit.

\paragraph{Step 6: Computation of the Dimension Bound}
We consider the $m$-th iterate of the cocycle, $\phi(m, \omega, x)$, $x \in \mathbb{R}^d$. By the cocycle property we have for $m,k\in \mathbb{N}$ and $x \in \mathbb{R}^d$
\begin{equation}
\label{eq:cocycle_mk}
    \phi(m+k, \omega, x) = \phi(k, \theta^m \omega, \phi(m, \omega, x)),
\end{equation}
and by applying the chain rule:
\begin{equation}
\label{eq:Dcocycle_mk}
    D\phi(m+k, \omega, x) = D\phi(k, \theta^m \omega, \phi(m, \omega, x)) \cdot D\phi(m, \omega, x).
\end{equation}

For a linear map $A \in \mathbb{R}^{d\times d}$ and $j \in \{1,\dots,d\}$, define the \emph{$j$-th exterior-power norm}
\begin{equation}\label{eq:omega_def}
    \|A\|_j := \sigma_1(A)\,\sigma_2(A)\cdots\sigma_j(A),
\end{equation}
where $\sigma_1(A)\ge\sigma_2(A)\ge\dots\ge\sigma_d(A)\ge 0$ are the singular values of $A$. We also set $\|A\|_0 := 1$ for convenience, so that $\sigma_j(A) = \|A\|_j/\|A\|_{j-1}$ for all $j \ge 1$.

By the functoriality of the exterior power, $\bigwedge^j(AB) = (\bigwedge^j A)(\bigwedge^j B)$, and the submultiplicativity of the operator norm, we have the fundamental inequality
\begin{equation}\label{eq:omega_submult}
    \|AB\|_j \le \|A\|_j\,\|B\|_j
    \qquad\text{for all }A,B\in\mathbb{R}^{d\times d}.
\end{equation}

\medskip
\noindent\textbf{6.1.\ Subadditivity and Fekete limits.}\quad
Define the expected log-growth of $\omega_j$ along the cocycle:
\begin{equation}\label{eq:Omega_def}
    \Omega_j^{(m)} := \mathbb{E}\!\left[\,\ln \sup_{w \in \mathcal{A}(\omega)} \|\bigl(D\phi(m,\omega,w)\bigr)\|_j\right], \qquad m \in \mathbb{N}^*.
\end{equation}
Combining the cocycle identity \eqref{eq:Dcocycle_mk} with the submultiplicativity \eqref{eq:omega_submult} gives
\[
    \|\bigl(D\phi(m+k,\omega,x)\bigr)\|_j
    \le
    \|\bigl(D\phi(k,\theta^m\omega,\phi(m,\omega,x))\bigr)\|_j
    \cdot
    \|\bigl(D\phi(m,\omega,x)\bigr)\|_j.
\]
Taking the supremum over $x \in \mathcal{A}(\omega)$ and using the invariance $\phi(m,\omega,\mathcal{A}(\omega))=\mathcal{A}(\theta^m\omega)$:
\[
    \sup_{x\in\mathcal{A}(\omega)} \|\bigl(D\phi(m+k,\omega,x)\bigr)\|_j
    \le
    \sup_{y \in \mathcal{A}(\theta^m\omega)} \|\bigl(D\phi(k,\theta^m\omega,y)\bigr) \|_j
    \;\cdot\;
    \sup_{x\in\mathcal{A}(\omega)} \|\bigl(D\phi(m,\omega,x)\bigr)\|_j.
\]
Taking logarithms, then expectations, and using the $\theta^m$-invariance of $\mathbb{P}$, we obtain the subadditivity
\begin{equation}
\label{eq:subadditive}
    \Omega_j^{(m+k)} \le \Omega_j^{(m)} + \Omega_j^{(k)}
    \qquad\text{for all }m,k\in\mathbb{N}^*.
\end{equation}
By the integrability assumption and Fekete's lemma, the following limit exists for each $j\in\{1,\dots,d\}$:
\begin{equation}\label{eq:Lambda_def}
    \Lambda_j := \lim_{m\to\infty} \frac{1}{m}\,\Omega_j^{(m)} = \inf_{m\ge 1}\frac{1}{m}\,\Omega_j^{(m)}.
\end{equation}

We set $\Lambda_0 := 0$. By Fekete's lemma, $\Lambda_j \le \Omega_j^{(1)}$. Moreover, since $\ln$ is monotone increasing and $\prod_{i=1}^j \sigma_i(D\phi(1,\omega,x)) > 0$ by the non-singularity assumption,
\begin{align}
    \Omega_j^{(1)} 
    &= \mathbb{E}\!\left[\ln \sup_{x \in \mathcal{A}(\omega)} \prod_{i=1}^j \sigma_i(D\phi(1,\omega,x))\right] 
    = \mathbb{E}\!\left[\sup_{x \in \mathcal{A}(\omega)} \sum_{i=1}^j \ln \sigma_i(D\phi(1,\omega,x))\right] \notag \\
    &\le \sum_{i=1}^j \mathbb{E}\!\left[\sup_{x \in \mathcal{A}(\omega)} \ln \sigma_i(D\phi(1,\omega,x))\right] 
    = \sum_{i=1}^j \lambda_i,
\end{align}
where the inequality uses $\sup_x(f_1(x) + \cdots + f_j(x)) \le \sup_x f_1(x) + \cdots + \sup_x f_j(x)$. Hence $\Lambda_j \le \sum_{i=1}^j \lambda_i$. In particular, if $\Lambda_j \ge 0$ then $\sum_{i=1}^j \lambda_i \ge 0$, so the transition index $j^*(\Lambda)$ (the largest $j \in \{0,\dots,d-1\}$ with $\Lambda_j \ge 0$) satisfies $j^*_\Lambda \le j^*$ the transition index defined in terms of the one-step exponents $\lambda_i$.

Now define the $m$-step sharpness exponents
\begin{equation}\label{eq:lambda_m_def}
    \lambda_i^{(m)} := \mathbb{E}\!\left[\sup_{x\in\mathcal{A}(\omega)}\ln\sigma_i\bigl(D\phi(m,\omega,x)\bigr)\right], \qquad i = 1,\dots,d.
\end{equation}
\noindent\textbf{6.3.\ Covering bound for the $m$-step map.}\quad
We apply the covering argument from Steps~1--3 to the $m$-step map 
$\phi(m,\omega,\cdot)$ in place of $\phi(1,\omega,\cdot)$, with the 
replacement $\theta\omega \leftarrow \theta^m\omega$. Steps~1 and~2 
(general covering and Taylor approximation) carry over without 
modification. In Step~3, the ellipsoid covering lemma 
(Lemma~\ref{lemma:ellipsoid_covering}) involves a free index $j$ at 
which the singular values are split into expanding directions 
$\sigma_1, \dots, \sigma_j$ and a contracting direction 
$\sigma_{j+1}$ that determines the covering radius. We choose $j = j^*_\Lambda$, where $j^*_\Lambda$ is the largest 
integer in $\{0,\dots,d-1\}$ such that $\Lambda_{j^*_\Lambda} \ge 0$. 
We now verify that this choice is well-defined and that the resulting 
covering argument is valid.

First, $j^*_\Lambda$ is well-defined: by Assumption~4, 
$\sum_{i=1}^{j^*+1}\lambda_i < 0$. Since 
$\Lambda_{j^*+1} \le \sum_{i=1}^{j^*+1}\lambda_i < 0$, and 
$\Lambda_0 = 0 \ge 0$, there exists at least one $j$ with 
$\Lambda_j \ge 0$ and at least one with $\Lambda_j < 0$, so 
$j^*_\Lambda$ is well-defined and satisfies 
$0 \le j^*_\Lambda \le j^*$.

Second, $\Lambda_{j^*_\Lambda + 1} < 0$ by definition of $j^*_\Lambda$. 
By Fekete's lemma, $\frac{1}{m}\,\Omega_{j^*_\Lambda+1}^{(m)} \to 
\Lambda_{j^*_\Lambda+1} < 0$. Since  for any $A \in \mathbb{R}^{d\times d}$ it holds $ \sigma_1(A) \ge \sigma_2(A) \ge \cdots \ge \sigma_d(A)$, we 
have $\sigma_k(A) \ge \sigma_{j^*_\Lambda+1}(A)$ for all 
$k \le j^*_\Lambda+1$, and therefore
\[
    \|A\|_{j^*_\Lambda+1} 
    = \prod_{k=1}^{j^*_\Lambda+1} \sigma_k(A) 
    \ge \sigma_{j^*_\Lambda+1}(A)^{j^*_\Lambda+1}.
\]
Taking logarithms and dividing by $j^*_\Lambda+1$:
\[
    \ln \sigma_{j^*_\Lambda+1}(A) 
    \le \frac{1}{j^*_\Lambda+1}\,\ln \|A\|_{j^*_\Lambda+1}.
\]
Applying this with $A = D\phi(m,\omega,x)$ and taking the supremum 
over $x \in \mathcal{A}(\omega)$:
\[
    \sup_{x \in \mathcal{A}(\omega)} \ln \sigma_{j^*_\Lambda+1}
    (D\phi(m,\omega,x)) 
    \;\le\; 
    \frac{1}{j^*_\Lambda+1}\, \ln \sup_{x \in \mathcal{A}(\omega)} 
    \|(D\phi(m,\omega,x))\|_{j^*_\Lambda+1},
\]
where we used the monotonicity of $\ln$ and the fact that 
$t \mapsto t^{1/(j^*_\Lambda+1)}$ is monotone increasing, so the 
supremum passes inside. Taking expectations and dividing by $m$:
\[
    \frac{1}{m}\,\lambda_{j^*_\Lambda+1}^{(m)} 
    \;\le\; 
    \frac{1}{j^*_\Lambda+1}\cdot\frac{1}{m}\,
    \Omega_{j^*_\Lambda+1}^{(m)} 
    \;\xrightarrow{m\to\infty}\; 
    \frac{\Lambda_{j^*_\Lambda+1}}{j^*_\Lambda+1} < 0,
\]
by Fekete's lemma. In particular, 
$\lambda_{j^*_\Lambda+1}^{(m)} \to -\infty$, which shows that 
$\sup_{x \in \mathcal{A}(\omega)} 
\sigma_{j^*_\Lambda+1}(D\phi(m,\omega,x))$ decays exponentially 
in $m$.
Therefore, the 
covering radius 
$\rho_1(\omega,m) = R\,\sup_{x}\sigma_{j^*_\Lambda+1}
(D\phi(m,\omega,x))$ shrinks with each iteration of the $m$-step map, 
and the iterative covering argument of Steps~4--5 produces a finite 
dimension bound.

With this choice, the covering number of each image ellipsoid satisfies
\begin{equation}
\label{eq:ellipsoid_cover_m}
    \mathcal{N}\bigl(D\phi(m,\omega,x_i)\,B(0,R),\;\rho_1(\omega,m)\bigr)
    \;\le\;
    3^d\,\frac{\prod_{k=1}^{j^*_\Lambda}\sigma_k\bigl(D\phi(m,\omega,x_i)\bigr)}
    {\bigl(\sup_{x\in\mathcal{A}(\omega)}\sigma_{j^*_\Lambda+1}(D\phi(m,\omega,x))\bigr)^{j^*_\Lambda}},
\end{equation}
where $\rho_1(\omega,m) := R\,\sup_{x\in\mathcal{A}(\omega)}\sigma_{j^*_\Lambda+1}(D\phi(m,\omega,x))$. To obtain a bound that is uniform over all centers $x_i \in \mathcal{A}(\omega)$, we use
\[
    \prod_{k=1}^{j^*_\Lambda}\sigma_k\bigl(D\phi(m,\omega,x_i)\bigr) = \|\bigl(D\phi(m,\omega,x_i)\bigr)\|_{j^*_\Lambda} \le \sup_{x\in\mathcal{A}(\omega)}\|\bigl(D\phi(m,\omega,x)\bigr)\|_{j^*_\Lambda}.
\]
Define the multiplier
\begin{equation}\label{eq:mu_m_def}
    \mu^{(m)}(\omega) := \frac{\sup_{x\in\mathcal{A}(\omega)} \|\bigl(D\phi(m,\omega,x)\bigr)\|_{j^*_\Lambda}}
    {\Bigl(\sup_{x\in\mathcal{A}(\omega)}\sigma_{j^*_\Lambda+1}\bigl(D\phi(m,\omega,x)\bigr)\Bigr)^{j^*_\Lambda}}.
\end{equation}
Then the one-step covering estimate (applied to the $m$-step map) gives
\begin{equation}\label{eq:covering_m}
    \mathcal{N}\bigl(\mathcal{A}(\theta^m\omega),\,\epsilon_1(\omega,m)\bigr)
    \le N_1(\omega,R)\,3^d\,\mu^{(m)}(\omega),
\end{equation}
where $\epsilon_1(\omega,m) = 2\rho_1(\omega,m) = 2R\,\sup_{x\in\mathcal{A}(\omega)}\sigma_{j^*_\Lambda+1}(D\phi(m,\omega,x))$.

\medskip
\noindent\textbf{6.4.\ Iterating the $m$-step map.}\quad
Define the sequence of radii for the $K$-fold iteration of the $m$-step map:
\begin{equation}\label{eq:eps_km}
    \epsilon_{K+1}(\omega,m)
    := 2\,\epsilon_K(\omega,m)\,\sup_{x\in\mathcal{A}(\theta^{mK}\omega)}\sigma_{j^*_\Lambda+1}\bigl(D\phi(m,\theta^{mK}\omega,x)\bigr),
    \qquad \epsilon_0 := R.
\end{equation}
By the inductive argument of Step~5 (with $\theta$ replaced by $\theta^m$):
\begin{equation}\label{eq:recursion_m}
    \mathcal{N}\bigl(\mathcal{A}(\theta^{mK}\omega),\,\epsilon_K(\omega,m)\bigr)
    \le N_1(\omega,R)\,\prod_{k=0}^{K-1} 3^d\,\mu^{(m)}(\theta^{mk}\omega).
\end{equation}
The radii iterate to
\begin{equation}\label{eq:radius_iterate_m}
    \epsilon_K(\omega,m) = R\,\prod_{k=0}^{K-1} 2\,\sup_{x\in\mathcal{A}(\theta^{mk}\omega)}\sigma_{j^*_\Lambda+1}\bigl(D\phi(m,\theta^{mk}\omega,x)\bigr).
\end{equation}

\medskip
\noindent\textbf{6.5.\ Ergodic limits.}\quad
Taking logarithms in \eqref{eq:recursion_m}, dividing by $K$, and applying the Birkhoff Ergodic Theorem to $\omega\mapsto\ln\mu^{(m)}(\omega)$ (with respect to $\theta^m$), we obtain
\begin{equation}\label{eq:num_limit}
    \limsup_{K\to\infty}\frac{1}{K}\ln\mathcal{N}\bigl(\mathcal{A}(\theta^{mK}\omega),\epsilon_K(\omega,m)\bigr)
    \le \mathbb{E}\bigl[\ln\mu^{(m)}(\omega)\bigr] + d\ln 3.
\end{equation}
From the definition of $\mu^{(m)}$:
\begin{equation}\label{eq:E_mu}
    \mathbb{E}\bigl[\ln\mu^{(m)}(\omega)\bigr] = \Omega_{j^*_\Lambda}^{(m)} - j^*_\Lambda\,\lambda_{j^*_\Lambda+1}^{(m)}.
\end{equation}
Similarly, from \eqref{eq:radius_iterate_m} and the Birkhoff Ergodic Theorem:
\begin{equation}\label{eq:den_limit}
    \lim_{K\to\infty}\frac{1}{K}\ln\epsilon_K(\omega,m)
    = \lambda_{j^*_\Lambda+1}^{(m)} + \ln 2.
\end{equation}

\medskip
\noindent\textbf{6.6.\ Sign condition and passage to the limit.}\quad
We need $\lambda_{j^*_\Lambda+1}^{(m)} + \ln 2 < 0$, i.e.,
\begin{equation}\label{eq:sign_condition}
    \lambda_{j^*_\Lambda+1}^{(m)} < -\ln 2.
\end{equation}
Since $\sigma_{j^*_\Lambda+1}(A)^{j^*_\Lambda+1} \le \omega_{j^*_\Lambda+1}(A)$ for any $A \in \mathbb{R}^{d \times d}$, we have the upper bound
\begin{equation}\label{eq:lambda_upper_sign}
    \lambda_{j^*_\Lambda+1}^{(m)} \le \frac{1}{j^*_\Lambda+1}\,\Omega_{j^*_\Lambda+1}^{(m)}.
\end{equation}
By Fekete's lemma, $\frac{1}{m}\,\Omega_{j^*_\Lambda+1}^{(m)} \to \Lambda_{j^*_\Lambda+1} < 0$, so $\Omega_{j^*_\Lambda+1}^{(m)} \to -\infty$ as $m \to \infty$. Combined with \eqref{eq:lambda_upper_sign}, this gives $\lambda_{j^*_\Lambda+1}^{(m)} \to -\infty$, and in particular \eqref{eq:sign_condition} holds for all $m$ sufficiently large.

For such $m$, the box-counting dimension satisfies 
\begin{equation}\label{eq:dim_bound_m}
    \overline{\dim}_M\,\mathcal{A}(\omega)
    \le \frac{\Omega_{j^*_\Lambda}^{(m)} - j^*_\Lambda\,\lambda_{j^*_\Lambda+1}^{(m)} + d\ln 3}{-\lambda_{j^*_\Lambda+1}^{(m)} - \ln 2}.
\end{equation}
Dividing numerator and denominator by $m$:
\begin{equation}\label{eq:dim_m}
    \overline{\dim}_M\,\mathcal{A}(\omega)
    \le \frac{\frac{1}{m}\,\Omega_{j^*_\Lambda}^{(m)} - j^*_\Lambda\,\frac{1}{m}\,\lambda_{j^*_\Lambda+1}^{(m)} + \frac{d\ln 3}{m}}{-\frac{1}{m}\,\lambda_{j^*_\Lambda+1}^{(m)} - \frac{\ln 2}{m}}.
\end{equation}
We now pass to the limit $m \to \infty$. By Fekete's lemma, $\frac{1}{m}\,\Omega_{j^*_\Lambda}^{(m)} \to \Lambda_{j^*_\Lambda}$. The terms $\frac{d\ln 3}{m}$ and $\frac{\ln 2}{m}$ vanish. It remains to identify the limit of $\frac{1}{m}\,\lambda_{j^*_\Lambda+1}^{(m)}$.

Since $\sigma_{j^*_\Lambda+1}(A) = \omega_{j^*_\Lambda+1}(A)/\omega_{j^*_\Lambda}(A)$, we have the two-sided bounds
\begin{equation}\label{eq:lambda_sandwich}
    \Omega_{j^*_\Lambda+1}^{(m)} - \Omega_{j^*_\Lambda}^{(m)} \;\le\; \lambda_{j^*_\Lambda+1}^{(m)} \;\le\; \Omega_{j^*_\Lambda+1}^{(m)} - \hat{\Omega}_{j^*_\Lambda}^{(m)},
\end{equation}
where the lower bound uses $\sup(f - g) \ge \sup f - \sup g$ and the upper bound uses $\sup(f - g) \le \sup f - \inf g$. By Fekete's lemma, $\frac{1}{m}(\Omega_{j^*_\Lambda+1}^{(m)} - \Omega_{j^*_\Lambda}^{(m)}) \to \Lambda_{j^*_\Lambda+1} - \Lambda_{j^*_\Lambda}$. By the Bounded Distortion assumption, $\frac{1}{m}(\Omega_{j^*_\Lambda}^{(m)} - \hat{\Omega}_{j^*_\Lambda}^{(m)}) \to 0$, and hence $\frac{1}{m}\,\hat{\Omega}_{j^*_\Lambda}^{(m)} \to \Lambda_{j^*_\Lambda}$, so that $\frac{1}{m}(\Omega_{j^*_\Lambda+1}^{(m)} - \hat{\Omega}_{j^*_\Lambda}^{(m)}) \to \Lambda_{j^*_\Lambda+1} - \Lambda_{j^*_\Lambda}$. By the squeeze theorem,
\begin{equation}\label{eq:lambda_limit_final}
    \lim_{m \to \infty} \frac{1}{m}\,\lambda_{j^*_\Lambda+1}^{(m)} = \Lambda_{j^*_\Lambda+1} - \Lambda_{j^*_\Lambda} =: \tilde{\lambda}_{j^*_\Lambda+1}.
\end{equation}
Substituting all limits into \eqref{eq:dim_m}, we obtain
\begin{equation}\label{eq:final_bound}
    \overline{\dim}_M\,\mathcal{A}(\omega)
    \;\le\;
    \frac{\Lambda_{j^*_\Lambda} - j^*_\Lambda\,\tilde{\lambda}_{j^*_\Lambda+1}}{-\tilde{\lambda}_{j^*_\Lambda+1}}
    = j^*_\Lambda + \frac{\Lambda_{j^*_\Lambda}}{\Lambda_{j^*_\Lambda} - \Lambda_{j^*_\Lambda+1}}
    =: D_s.
\end{equation}
We now show that $D_s \le \dimS$, where 
$\dimS:= j^* + \frac{\sum_{i=1}^{j^*}\lambda_i}{|\lambda_{j^*+1}|}$ 
denotes the bound with one-step exponents. 
Define $g(a,b) := \frac{a}{a-b}$ for $a \ge 0 > b$, so that 
$D_s = j^*_\Lambda + g(\Lambda_{j^*_\Lambda}, \Lambda_{j^*_\Lambda+1})$.
Since $b < 0$ and $a \ge 0$:
\[
    \frac{\partial g}{\partial a} = \frac{-b}{(a-b)^2} > 0, 
    \qquad 
    \frac{\partial g}{\partial b} = \frac{a}{(a-b)^2} \ge 0,
\]
so $g$ is monotonically increasing in both arguments, and 
$g(a,b) < 1$ since $a < a - b$ (because $b < 0$).

We distinguish two cases.

\textbf{Case 1: $j^*_\Lambda = j^*$.}\quad
Since $\Lambda_j \le \sum_{i=1}^j \lambda_i$ for all $j$, the 
monotonicity of $g$ in the first argument gives
\[
    g(\Lambda_{j^*},\, \Lambda_{j^*+1}) 
    \;\le\; 
    g\!\left(\sum_{i=1}^{j^*} \lambda_i,\; 
    \Lambda_{j^*+1}\right).
\]
By Assumption~4, $\sum_{i=1}^{j^*+1}\lambda_i < 0$, and since 
$\Lambda_{j^*+1} \le \sum_{i=1}^{j^*+1}\lambda_i < 0$, the 
monotonicity of $g$ in the second argument yields
\[
    g\!\left(\sum_{i=1}^{j^*} \lambda_i,\; 
    \Lambda_{j^*+1}\right) 
    \;\le\; 
    g\!\left(\sum_{i=1}^{j^*} \lambda_i,\; 
    \sum_{i=1}^{j^*+1} \lambda_i\right) 
    = \frac{\sum_{i=1}^{j^*} \lambda_i}{|\lambda_{j^*+1}|}.
\]
Therefore $D_s = j^* + g(\Lambda_{j^*}, \Lambda_{j^*+1}) 
\le j^* + \frac{\sum_{i=1}^{j^*}\lambda_i}{|\lambda_{j^*+1}|} 
= D_\lambda$.

\textbf{Case 2: $j^*_\Lambda < j^*$.}\quad
Since $g(\Lambda_{j^*_\Lambda}, \Lambda_{j^*_\Lambda+1}) < 1$, 
we have $D_s < j^*_\Lambda + 1 \le j^*$. On the other hand, 
$D_\lambda = j^* + \frac{\sum_{i=1}^{j^*}\lambda_i}{|\lambda_{j^*+1}|} 
\ge j^*$, since $\sum_{i=1}^{j^*}\lambda_i \ge 0$ by Assumption~4. 
Therefore $D_s < j^* \le D_\lambda$.

\medskip
Combining both cases, we conclude that for 
$\mathbb{P}$-almost every $\omega$:
\begin{equation}\label{eq:full_chain}
    \overline{\dim}_M(\mathcal{A}(\omega)) 
    \;\le\; D_s 
    \;\le\; D_\lambda 
    \;=\; j^* + \frac{\sum_{i=1}^{j^*}\lambda_i}{|\lambda_{j^*+1}|}
    \;<\; \dim_{\mathrm{S}}\mathcal{A}.
\end{equation}
The bound $D_s$ is strictly sharper than $\dimS$ whenever 
$\Lambda_j < \sum_{i=1}^j \lambda_i$ for some 
$j \in \{j^*_\Lambda, j^*_\Lambda+1\}$, which occurs when 
different points on the attractor maximize different singular 
values $\sigma_i$. This completes Step~6.

\paragraph{Step 7: Transition to Arbitrary Radii}
The covering estimates obtained in Step~6 are formulated along the 
discrete sequence of radii $\{\epsilon_K(\omega,m)\}_{K \in \mathbb{N}}$. 
In this step, we show that the dimension bound extends to arbitrary 
radii $\epsilon \to 0$, thereby establishing an upper bound on the 
upper Minkowski dimension.

Fix $m \in \mathbb{N}$ sufficiently large such that 
$\lambda_{j^*_\Lambda+1}^{(m)} + \ln 2 < 0$ (which is possible by 
Step~6.6). Let $\{\epsilon_K(\omega,m)\}_{K \in \mathbb{N}}$ be the 
sequence of radii defined in \eqref{eq:eps_km}. By 
\eqref{eq:den_limit}, the Birkhoff Ergodic Theorem gives
\begin{equation}\label{eq:radii_exp_rate}
    \lim_{K \to \infty} \frac{1}{K} \ln \epsilon_K(\omega,m) 
    = \lambda_{j^*_\Lambda+1}^{(m)} + \ln 2 =: \ell(m) < 0 
    \qquad \mathbb{P}\text{-a.s.}
\end{equation}
In particular, $\epsilon_K(\omega,m) \to 0$ exponentially as 
$K \to \infty$, and for $\mathbb{P}$-a.e.\ $\omega$ the sequence 
$\{\epsilon_K(\omega,m)\}_{K}$ is eventually strictly decreasing.

We now work with a fixed realization $\omega$ in the full-measure 
set where \eqref{eq:radii_exp_rate} and the covering estimate 
\eqref{eq:recursion_m} both hold. The estimates in Step~6 bound 
$\mathcal{N}(\mathcal{A}(\theta^{mK}\omega), \epsilon_K(\omega,m))$. 
Since $\theta^{mK}$ preserves $\mathbb{P}$ and the entire argument 
from Steps~1--6 applies with $\omega$ replaced by any $\omega'$ in 
the underlying full-measure set, the same estimates hold with 
$\theta^{mK}\omega$ replaced by $\omega$. That is, for 
$\mathbb{P}$-a.e.\ $\omega$ there exists a sequence of radii 
$\epsilon_K(\omega,m) \to 0$ such that
\begin{equation}\label{eq:covering_own_fiber}
    \limsup_{K \to \infty} \frac{1}{K} 
    \ln \mathcal{N}(\mathcal{A}(\omega), \epsilon_K(\omega,m)) 
    \;\le\; 
    \mathbb{E}[\ln \mu^{(m)}(\omega)] + d\ln 3.
\end{equation}

For any $\epsilon > 0$ sufficiently small, there exists 
$K = K(\epsilon)$ such that 
$\epsilon \in [\epsilon_{K+1}(\omega,m),\, \epsilon_K(\omega,m))$. 
Since covering numbers are monotonically non-increasing in the 
radius:
\begin{equation}\label{eq:sandwich_covering}
    \mathcal{N}(\mathcal{A}(\omega), \epsilon_K(\omega,m)) 
    \;\le\; 
    \mathcal{N}(\mathcal{A}(\omega), \epsilon) 
    \;\le\; 
    \mathcal{N}(\mathcal{A}(\omega), \epsilon_{K+1}(\omega,m)).
\end{equation}
Taking logarithms and dividing by $-\ln\epsilon > 0$:
\begin{equation}\label{eq:sandwich_dim}
    \frac{\ln \mathcal{N}(\mathcal{A}(\omega), 
    \epsilon_K(\omega,m))}{-\ln \epsilon_{K+1}(\omega,m)} 
    \;\le\; 
    \frac{\ln \mathcal{N}(\mathcal{A}(\omega), \epsilon)}{-\ln \epsilon} 
    \;\le\; 
    \frac{\ln \mathcal{N}(\mathcal{A}(\omega), 
    \epsilon_{K+1}(\omega,m))}{-\ln \epsilon_K(\omega,m)},
\end{equation}
where we used $\epsilon_{K+1}(\omega,m) \le \epsilon < 
\epsilon_K(\omega,m)$ to bound $-\ln\epsilon$ from below and above 
in the denominators.

It remains to show that the lower and upper bounds in 
\eqref{eq:sandwich_dim} have the same $\limsup$. By 
\eqref{eq:radii_exp_rate}, $\frac{1}{K}\ln\epsilon_K(\omega,m) 
\to \ell(m) < 0$, and therefore
\begin{equation}
    \frac{\ln \epsilon_{K+1}(\omega,m)}{\ln \epsilon_K(\omega,m)} 
    = \frac{\frac{1}{K+1}\ln\epsilon_{K+1}(\omega,m)}
    {\frac{1}{K}\ln\epsilon_K(\omega,m)} \cdot \frac{K+1}{K} 
    \;\xrightarrow{K \to \infty}\; 
    \frac{\ell(m)}{\ell(m)} \cdot 1 = 1.
\end{equation}
Consequently, replacing $-\ln\epsilon_{K+1}$ by $-\ln\epsilon_K$ 
(or vice versa) in the denominators of \eqref{eq:sandwich_dim} does 
not affect the $\limsup$. Both the lower and upper bounds converge 
to the same value, and we conclude
\begin{equation}\label{eq:dim_arbitrary_radii}
    \limsup_{\epsilon \to 0} 
    \frac{\ln \mathcal{N}(\mathcal{A}(\omega), \epsilon)}{-\ln \epsilon} 
    \;=\; 
    \limsup_{K \to \infty} 
    \frac{\ln \mathcal{N}(\mathcal{A}(\omega), 
    \epsilon_K(\omega,m))}{-\ln \epsilon_K(\omega,m)}.
\end{equation}
Combining \eqref{eq:dim_arbitrary_radii} with the covering estimate 
\eqref{eq:covering_own_fiber} and the radii asymptotics 
\eqref{eq:radii_exp_rate}, we obtain for each fixed $m$ sufficiently 
large:
\begin{equation}\label{eq:dim_fixed_m}
    \overline{\dim}_M(\mathcal{A}(\omega)) 
    \;=\; 
    \limsup_{\epsilon \to 0} 
    \frac{\ln \mathcal{N}(\mathcal{A}(\omega), \epsilon)}{-\ln\epsilon} 
    \;\le\; 
    \frac{\mathbb{E}[\ln\mu^{(m)}(\omega)] + d\ln 3}
    {-\lambda_{j^*_\Lambda+1}^{(m)} - \ln 2}.
\end{equation}
Passing to the limit $m \to \infty$ as in Step~6.6 yields
\[
    \overline{\dim}_M(\mathcal{A}(\omega)) \;\le\; \dimS
\]
for $\mathbb{P}$-almost every $\omega$. This completes the proof. 
\end{proof}

\begin{theorem}{Generalization via Sharpness Dimension}~
\label{thm:generalization_DLS_copy}

Let $S = \{z_1, \dots, z_n\} \sim \datadist$ be a dataset of size $n$. Let $(\Omega, \mathcal{F}, \mathbb{P}, \theta, \phi)$ be a discrete-time RDS according to Dfn~\ref{def:rds} such that Assump.~\ref{assum:regular_dynamics} holds.
Under Assumps.~\ref{assum:bounded_loss} and~\ref{assum:lipschitz}, there exists a constant $C>0$ s.t. with probability at least $1 - \zeta - \gamma$ over the joint draw $(S, \omega) \sim \datadist \otimes \mathbb{P}$, there exists $\delta_{n,\gamma} > 0$ such that for all $0 < \delta < \delta_{n,\gamma}$,
\begin{align*}
\mathcal{G}_S(\mathcal{A}(\omega))
&\leq 2 L \delta
+ 2 B \sqrt{\frac{4\, \dim_{\mathrm{S}}\mathcal{A}_S \> \log(1/\delta)}{n}} \\
&\quad + \frac{I_{\infty}(\mathcal{A}_S(\omega), S) + \log(1/\zeta)}{\sqrt{n}}
+ \frac{C B^2}{\sqrt{n}}.
\end{align*}
We recall that $\mathcal{G}_S(\mathcal{A}(\omega))$ denotes the worst-case generalization gap (see \eqref{eq:worst_case_gap}) and $I_{\infty}(\mathcal{A}_S(\omega), S)$ (see Dfn.~\ref{def:total_mutual_information})  denotes the total mutual information between the random pullback attractor $\mathcal{A}_S(\omega)$ and $S$. 
\end{theorem}
\begin{proof}
\label{proof:main}
The result is an immediate consequence of Cor.~\ref{cor:uniform_generalization_bound} and Thm.~\ref{thm:KY_bound_general}.
\end{proof}

\begin{theorem}[Stability $D_S$ Bound]
\label{thm:stability_bound_D_S}
Suppose the loss function $\ell$ is $B$-bounded and $L$-Lipschitz, and Assump.~\ref{ass:random-trajectory-stability} holds. For each dataset $S \in \mathcal{Z}^n$, let $(\Omega, \mathcal{F}, \mathbb{P}, \theta, \phi)$ be a $C^2$ discrete-time RDS (per Dfn.~\ref{def:rds}) possessing a unique compact random pullback attractor $\mathcal{A}_S(\omega)$. We assume the following conditions hold:

\begin{enumerate}[leftmargin=*, label=(\roman*)]
    \item \textbf{Non-Singularity:} For $\mathbb{P}$-a.e. $\omega$, the Jacobian is non-singular on the attractor:
    \begin{align*}
        \inf_{x \in \mathcal{A}_S(\omega)} \sigma_d(D \phi(1,\omega,x)) > 0.
    \end{align*}

    \item \textbf{Integrability:} The first and second derivatives of the map satisfy:
    \begin{align*}
        \mathbb{E}\left[\sup_{x \in \mathcal{A}_S(\omega)} \ln^+ \|D\phi_S(1,\omega,x)\|\right] < \infty \quad \text{and} \quad
        \mathbb{E}\left[\sup_{x \in \mathcal{A}_S(\omega)} \ln^+ \|D^2\phi_S(1,\omega,x)\|\right] < \infty,
    \end{align*}
    where $\ln^+(x) := \max \{ 0, \ln x \}$.

    \item \textbf{Transition Index:} There exists an integer $j^* \in \{1, \dots, d-1\}$ such that the global sharpness values $\lambda_i$ (see Dfn.~\ref{def:local_sharpness_unified}) satisfy:
    \[
    \sum_{i=1}^{j^*} \lambda_i \geq 0 \quad \text{and} \quad \sum_{i=1}^{j^*+1} \lambda_i < 0.
    \]
     .
\item \textbf{Bounded Distortion}: For $A \in \mathbb{R}^{d \times d}$ and $j \in \{1, \dots, d\}$, define
\begin{equation}
    \|A\|_j := \sigma_1(A)\cdots\sigma_j(A),
\end{equation}
where $\sigma_1(A) \ge \cdots \ge \sigma_d(A)$ are the singular values of $A$. Equivalently, $\|A\|_j$ is the maximal expansion factor of $A$ on $j$-dimensional volumes.

We assume that the spatial variation of $\|D\phi(m,\omega,\cdot)\|_j$ over the attractor is subexponential in $m$: for each $j \in \{1, \dots, d\}$,
\begin{equation}
    \lim_{m \to \infty} \frac{1}{m}\, \mathbb{E}\!\left[\sup_{x \in \mathcal{A}(\omega)} \ln \|D\phi(m, \omega, x)\|_j \;-\; \inf_{x \in \mathcal{A}(\omega)} \ln \|D\phi(m, \omega, x)\|_j\right] = 0.
\end{equation}
    In other words, the maximal and minimal $j$-volume growth rates over $\mathcal{A}(\omega)$ agree at exponential scale.
\end{enumerate}

Furthermore, assume $\beta_n^{-2/3}$ is an integer divisor of $n$. Then, there exists $\delta_n > 0$ such that for all $0 < \delta < \delta_n$:
\begin{align*}
    \mathbb{E} \bigg[\sup_{w \in \mathcal{W}_{S,U}} \left(\mathcal{R}(w) - \widehat{\mathcal{R}}_S(w) \right) \bigg] \leq 2 \mathbb{E} \bigg[ \frac{B}{n} + \delta L + \beta_n^{1/3} \left( 1 + B \sqrt{4 \dimS \log \frac{1}{\delta}} \right) \bigg].
\end{align*}
\end{theorem}
\begin{proof}
    The proof follows from Thm.~\ref{thm:fractal-dim-bound} and Thm.~\ref{thm:KY_bound_general}.
\end{proof}

\section{Implementation Details and Computational Complexity}

\paragraph{Lanczos and SLQ} For all SLQ-based estimators, each Lanczos run uses one Rademacher probe vector and one minibatch Hessian operator. The
operator is fixed throughout the Lanczos iterations of that run, so the Krylov subspace and Gaussian quadrature
interpretation are well defined. We resample only between independent runs, average the resulting quadrature measures or
their histogram/smoothed density estimates, and then compute SD-SLQ, SD-KDE, or SD-PS from this averaged spectral
measure. This procedure estimates the expectation over minibatch Hessian operators and probe vectors without mixing
Hessian-vector products from different operators inside a single Lanczos run.

More explicitly, let \(N\) be the number of parameters and let \(R_{\mathrm{SLQ}}\) be the number of independent SLQ
runs. In run \(i\), we sample a minibatch \(\mathcal{B}_i\) and define the corresponding Hessian operator
\begin{equation}
    H_i := \nabla_w^2 \ell(w^\star;\mathcal{B}_i).
\end{equation}
Given an independent Rademacher probe \(v_i \in \{\pm 1\}^N\), we set \(q_{i,1}=v_i/\|v_i\|\) and run \(m\) Lanczos
steps on the fixed operator \(H_i\). This produces a basis \(Q_i=[q_{i,1},\ldots,q_{i,m}]\) and a symmetric tridiagonal
matrix \(T_i\), with diagonal entries \(\alpha_{i,1},\ldots,\alpha_{i,m}\) and off-diagonal entries \(\beta_{i,1},\ldots,\beta_{i,m-1}\), such that
\begin{equation}
    H_i Q_i
    =
    Q_i T_i
    +
    \beta_{i,m} q_{i,m+1} e_m^\top .
\end{equation}
Here \(e_m\) denotes the \(m\)-th standard basis vector in \(\mathbb{R}^m)\).
Let
\begin{equation}
    T_i = U_i \operatorname{diag}(\widetilde{\alpha}_{i,1},\ldots,\widetilde{\alpha}_{i,m}) U_i^\top
\end{equation}
be the eigendecomposition of this tridiagonal matrix. The eigenvalues \(\widetilde{\alpha}_{i,\ell}\) of \(T_i\) are the Ritz values, i.e. Lanczos approximations to the eigenvalues of \(H_i\). The associated columns of \(U_i\) are the Ritz vectors, and the corresponding Gaussian quadrature weights are given by their squared first components:
\begin{equation}
    \omega_{i,\ell} = (U_i)_{1\ell}^2,
    \qquad \ell=1,\ldots,m.
\end{equation}
Thus, for a test function \(f\),
\begin{equation}
    q_{i,1}^\top f(H_i) q_{i,1}
    \approx
    e_1^\top f(T_i)e_1
    =
    \sum_{\ell=1}^{m} \omega_{i,\ell} f(\widetilde{\alpha}_{i,\ell}).
\end{equation}
Because \(q_{i,1}\) is a normalized Rademacher probe, multiplying by \(N\) gives an unbiased trace-scale estimate in
expectation over the probe. Averaging across the independently sampled minibatch Hessians and probes yields the
empirical spectral measure
\begin{equation}
    \widehat{\nu}_{H}
    =
    \frac{N}{R_{\mathrm{SLQ}}}
    \sum_{i=1}^{R_{\mathrm{SLQ}}}
    \sum_{\ell=1}^{m}
    \omega_{i,\ell}\,\delta_{\widetilde{\alpha}_{i,\ell}} .
\end{equation}
Here \(\delta_x\) denotes the Dirac measure at \(x\).
This is the object averaged by our SLQ procedure: not full spectra, but the quadrature measures induced by Ritz values
and weights from independent fixed-operator Lanczos runs.
The histogram estimator used for SD-SLQ is obtained by binning this measure. For a bin \(I_b=[a_b,a_{b+1})\) of width
\(\Delta_b=a_{b+1}-a_b\), we set
\begin{equation}
    \widehat{\rho}_{\mathrm{hist}}(\alpha)
    =
    \frac{N}{R_{\mathrm{SLQ}}\,\Delta_b}
    \sum_{i=1}^{R_{\mathrm{SLQ}}}
    \sum_{\ell=1}^{m}
    \omega_{i,\ell}\mathbf{1}\{\widetilde{\alpha}_{i,\ell}\in I_b\},
    \qquad \alpha\in I_b.
\end{equation}
The smoothed estimator used for SD-KDE replaces the bin indicator by a kernel \(K_h\) with bandwidth \(h\):
\begin{equation}
    \widehat{\rho}_{\mathrm{kde}}(\alpha)
    =
    \frac{N}{R_{\mathrm{SLQ}}}
    \sum_{i=1}^{R_{\mathrm{SLQ}}}
    \sum_{\ell=1}^{m}
    \omega_{i,\ell} K_h(\alpha-\widetilde{\alpha}_{i,\ell}).
\end{equation}
To compute the Sharpness Dimension estimators, we apply the one-step Jacobian map
\begin{equation}
    g_\eta(\alpha)=\log\!\left(|1-\eta\alpha|+\varepsilon\right)
\end{equation}
to the Ritz values and estimate the corresponding push-forward density. SD-SLQ and SD-KDE apply the Sharpness Dimension
formula to the histogram or smoothed density of these transformed values, respectively. SD-PS, described next, uses the same
smoothed-density viewpoint but first converts the density into equal-mass pseudo-eigenvalues before applying the
finite-dimensional weighted sharpness dimension formula.

\paragraph{Pseudo-spectrum SD (SD-PS)}
\label{app:sd-ps}
We also report a pseudo-spectrum version of the SLQ estimator, denoted SD-PS, which discretizes the smoothed spectral density before applying the same sharpness-dimension formula as in Dfn.~\ref{def:local_sharpness_unified}. Let $\widehat{\rho}(\alpha)$ be the smoothed SLQ estimate of the Hessian spectral density for a model with $N$ parameters, normalized so that $\int \widehat{\rho}(\alpha)\,d\alpha=N$. We form its cumulative mass function $\widehat{F}(\alpha)=\int_{-\infty}^{\alpha}\widehat{\rho}(s)\,ds$ and choose equal-mass quantiles
\begin{equation}
    \widehat{\alpha}_r
    :=
    \widehat{F}^{-1}\!\left(\frac{r+1/2}{M}N\right),
    \qquad r=0,\ldots,M-1,
\end{equation}
where $M \le N$ is the number of pseudo-eigenvalues used in the discretization. Each pseudo-eigenvalue carries weight $N/M$. We then map these values through the one-step GD Jacobian spectrum,
\begin{equation}
    \widehat{z}_r = \log\!\left(|1-\eta \widehat{\alpha}_r|+\varepsilon\right),
\end{equation}
sort the values and reindex them so that \(\widehat{z}_0 \ge \cdots \ge \widehat{z}_{M-1}\), and compute the weighted Sharpness Dimension: 
if $j$ is the largest index for which $\sum_{r=0}^{j} (N/M)\widehat{z}_r \ge 0$, then
\begin{equation}
    \widehat{\dim}_{\mathrm{S}}^{\mathrm{PS}}
    =
    \sum_{r=0}^{j}\frac{N}{M}
    +
    \frac{\sum_{r=0}^{j}(N/M)\widehat{z}_r}{|\widehat{z}_{j+1}|},
\end{equation}
with the usual truncation to $[0,N]$ and the boundary conventions from Dfn.~\ref{def:local_sharpness_unified}. In our
GPT-2 experiments we use this estimator as an alternative discretization of the same smoothed SLQ measure used for
SD-KDE. Thus SD-PS differs from SD-KDE only in how the density is converted into the ordered log-singular-value
spectrum: SD-KDE integrates the density directly, whereas SD-PS first constructs equal-mass pseudo-eigenvalues and then
applies the finite-dimensional weighted Sharpness Dimension formula.

\paragraph{Complexity analysis} 
Regarding the SD computation, for a network with $N$ parameters and $k$ minibatches, assume the cost of one Hessian–vector product (HVP) is $C_{\text{hvp}}$. Explicit Hessian computation has memory complexity $\Theta(N^2)$ and worst-case runtime complexity $\mathcal{O}(N \cdot C_{\text{hvp}})$. Singular value or eigenvalue decompositions have memory complexity $\Theta(N^2)$ and runtime complexity $\Theta(N^3)$. In this case, since SD computation requires $\Theta(N)$ memory and runtime, the overall algorithm has $\mathcal{O}(k N C_{\text{hvp}} + k N^3)$ runtime complexity and $\Theta(N^2)$ memory complexity. In the approximate SLQ variant, assuming $m$ iterations per Lanczos run, each run has runtime complexity $\Theta(m \cdot C_{\text{hvp}})$ for HVPs and $\Theta(m^2 N)$ for reorthogonalization, with $\Theta(m N)$ memory complexity. Hence, the SLQ-based algorithm has total runtime complexity $\Theta(k m C_{\text{hvp}} + k m^2 N)$ and memory complexity $\Theta(m N)$. For typical settings where $m \ll N$, both runtime and memory complexity are significantly reduced.

\section{Additional Results}
\label{app:grok}
\paragraph{Grokking for the 3-layer MLP}
Below in Fig~\ref{fig:grokking-3-layer2} we present another experiment, with different seeds and hyper-parameters in the grokking setting identical to the main paper but for a 3-layer MLP with 32 hidden features instead.  
Similarly, we use 100 uniformly spaced checkpoints and  ReLU activation, trained with SGD using only weight decay. In particular, we observe an interesting phenomenon in the first plot, where our dimension increases while test accuracy is increasing slowly, sharply decreases while the test accuracy is increasing sharply (grokking), and then increases again afterwards when test accuracy reaches 100\%. In contrast, the other measures decrease monotonically. This behaviour aligns with the theoretical motivation of our dimension, which targets the edge-of-stability regime, and accurately reflects the grokking phase transitions. A similar phase transition is visible in Plot 3, while Plot 4 reflects the grokking experiment from the main paper. In Plot 2, focusing on the sharp grokking region reveals a similarly sharp decrease in SD.

\begin{figure*}[h!]
            \begin{center}
    \includegraphics[width=\linewidth]{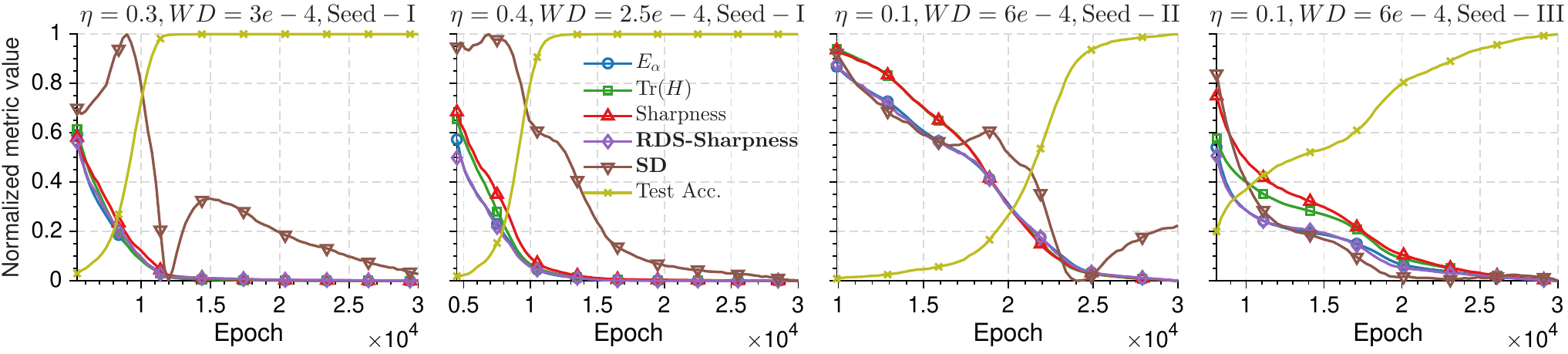}
            \end{center}
            \vspace{-1.5mm}
            \caption{\textbf{Grokking analysis} for different learning rates ($\eta$), weight decay ($WD$) and seeds 3-layer MLP with ReLU activation and no momentum. Note that the \emph{suddenness } of the grokking behavior is best captured in the complexity measures we introduce: RDS-Sharpness and Sharpness Dimension (SD).}
            \label{fig:grokking-3-layer2}
\end{figure*}

\paragraph{Hessian Spectra \& RDS Sharpness Spectrum.}
Figure~\ref{fig:RDSSpectra} illustrates our spectral estimators for Hessian spectral density and the RDS sharpness spectrum on GPT-2 across SGD, SGD with momentum, and AdamW. The left image in each pair visualizes the expectation of minibatch Hessian spectral density over eigenvalues \(\alpha\). The corresponding right images show the push-forward of this density through the transformation \(\alpha \mapsto \log|1-\eta \alpha|\), i.e., the transformed spectral density underlying the RDS Sharpnesses of Order \(k\), \(\lambda_k\). We observe that, across all three optimizers, a large fraction of the Hessian spectrum is concentrated near \(\alpha=0\). This mass near \(\alpha=0\) produces a peak near \(0\) in the RDS sharpness spectrum and corresponds to neutral or nearly neutral directions (i.e., expanding and contracting with almost 0 log-singular values). In contrast, the isolated spikes, positive tails, and negative curvature components of the Hessian spectrum produce the contractive and expansive directions of the RDS Sharpness spectrum. Positive eigenvalues less than \(2/\eta\) yield negative RDS sharpness values, corresponding to the contractive directions. Positive eigenvalues that are larger than \(2/\eta\) and all negative eigenvalues yield positive RDS sharpness values, corresponding to expansive directions instead. Hence, this visualization clarifies why the Sharpness Dimension depends on the full Hessian spectrum rather than only the top Hessian eigenvalue or a small part of it: SD is determined by the balance between the positive RDS sharpness directions, the near-neutral bulk, and the negative tail. Moreover, the estimated Hessian densities are consistent with prior observations on neural-network and transformer Hessian spectra~\cite{ghorbani2019investigation, zhang2024transformers}, providing additional evidence that our SLQ-based procedure captures the relevant spectral structure.

\begin{figure*}
            \begin{center}
    \includegraphics[width=\linewidth]{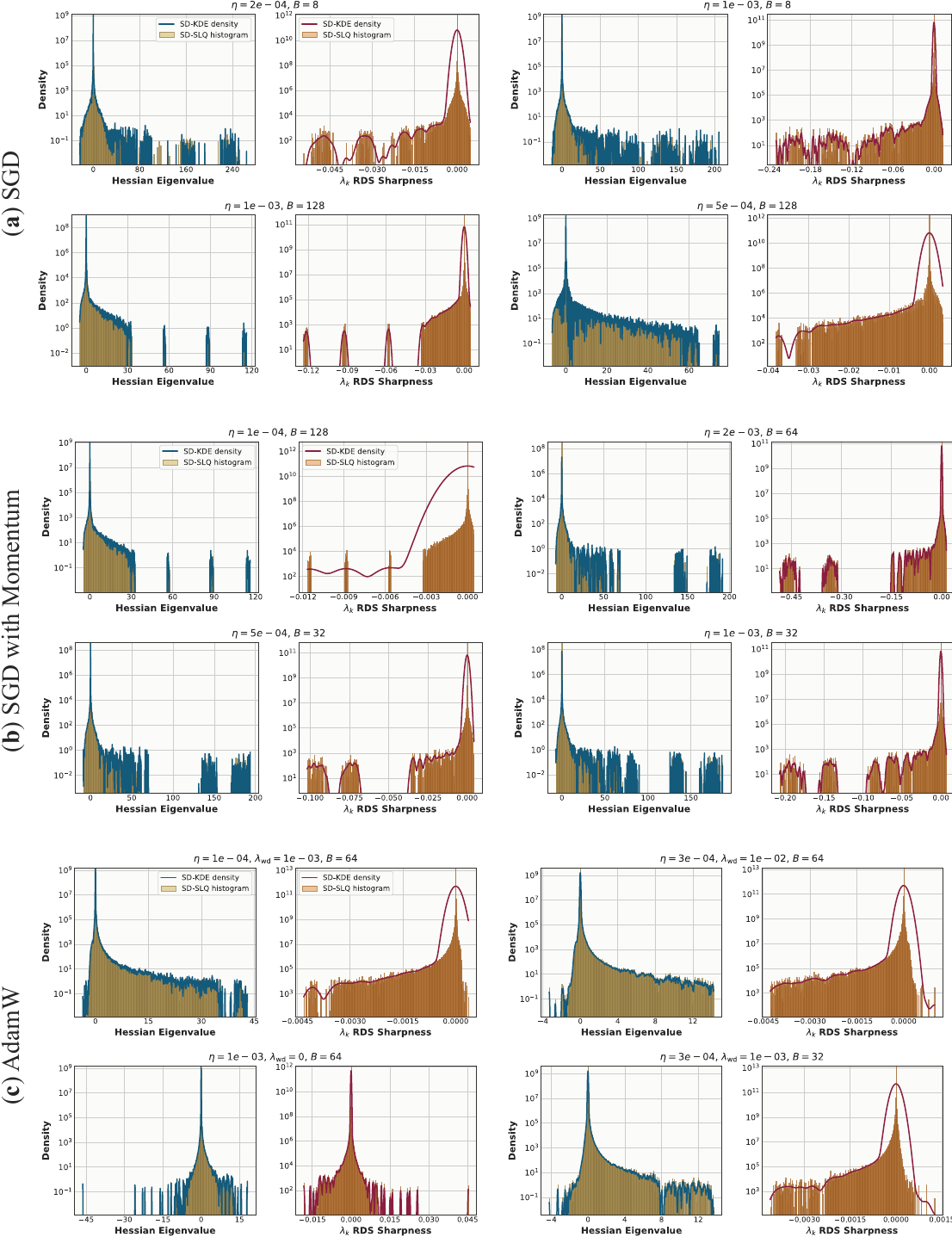}
            \end{center}
            \vspace{-1.5mm}
            \caption{\textbf{Hessian Spectra \& RDS Sharpness Spectrum.} For selected GPT-2 runs trained with (\textbf{a}) SGD, (\textbf{b}) SGD with momentum, and (\textbf{c}) AdamW, we show the SLQ-based histogram estimate (SD-SLQ) and the kernel-smoothed estimate (SD-KDE) for both the raw Hessian spectrum (\textbf{left} in each pair) and the transformed RDS sharpness spectrum \(\log|1-\eta \alpha|\) (\textbf{right} in each pair). Together, the panels show how optimizer and hyperparameter choices affect both the Hessian spectrum and its induced RDS sharpness spectrum. These examples show that the SLQ histogram and the corresponding kernel-smoothed estimate provide consistent views of the underlying spectrum across a range of training configurations.\vspace{-4mm}}
            \label{fig:RDSSpectra}
\end{figure*}

\end{document}